\documentclass[english]{article}
\usepackage[margin=1in]{geometry}
\usepackage[utf8]{inputenc}
 \usepackage[T1]{fontenc}    
\usepackage{float}
\usepackage{amsmath}
\usepackage{amsthm}
\usepackage{amssymb}
\usepackage{physics}
\usepackage{xcolor}
\usepackage{graphicx}
\usepackage{pgfplots}
\pgfplotsset{compat=1.7}
\usepackage{subcaption}
\usepackage{natbib}
\pgfplotsset{yticklabel style={text width=3em,align=right}}
\usepackage{bm}
\usepackage{thmtools}
\usepackage{url}
\PassOptionsToPackage{hyphens}{url}\usepackage{hyperref}
\usepackage{cleveref}

\makeatletter
\allowdisplaybreaks

\newcommand{\Lap}{\mathrm{Lap}}
\newcommand{\Gumb}{\mathrm{Gumb}}
\newcommand{\ct}{\mbox{count}}
\newcommand{\OPT}{\mathrm{OPT}}
\newcommand{\G}{G}

\floatstyle{ruled}
\newfloat{algorithm}{tbp}{loa}
\providecommand{\algorithmname}{Algorithm}
\floatname{algorithm}{\protect\algorithmname}

\newtheorem{theorem}{Theorem}
\newtheorem{definition}[theorem]{Definition}
\newtheorem{lemma}[theorem]{Lemma}
\newtheorem{remark}[theorem]{Remark}
\newtheorem{claim}[theorem]{Claim}
\usepackage{algpseudocode}

\makeatother

\DeclareMathOperator{\Ex}{\mathrm{E}}

\begin{document}
\title{Streaming Submodular Maximization with Differential Privacy}
\author{Anamay Chaturvedi\footnotemark[1], Huy L\^{e} Nguy\~{\^{e}}n\footnotemark[1], Thy Nguyen\thanks{Equal contribution. All three authors were supported in part by NSF CAREER grant CCF-1750716 and NSF grant CCF-1909314.}}
\date{
\texttt{\{chaturvedi.a,hu.nguyen,nguyen.thy2\} @northeastern.edu}\\[2ex]
Khoury College of Computer Sciences,\\
Northeastern University\\[2ex]
\today
}

\maketitle
\begin{abstract}
    In this work, we study the problem of privately maximizing a submodular function in the streaming setting. Extensive work has been done on privately maximizing submodular functions in the general case when the function depends upon the private data of individuals. However, when the size of the data stream drawn from the domain of the objective function is large or arrives very fast, one must privately optimize the objective within the constraints of the streaming setting. We establish fundamental differentially private baselines for this problem and then derive better trade-offs between privacy and utility for the special case of decomposable submodular functions. A submodular function is decomposable when it can be written as a sum of submodular functions; this structure arises naturally when each summand function models the utility of an individual and the goal is to study the total utility of the whole population as in the well-known Combinatorial Public Projects Problem. Finally, we complement our theoretical analysis with experimental corroboration.
\end{abstract}

\section{Introduction}

Consider the task of a service provider that trains machine learning models for a set of users, e.g., platforms such as  Amazon Sagemaker and Microsoft Azure. In many cases, collecting features can be costly and  the service provider has to select up to a limited number of features for their models. Since each set of features benefits different users differently, i.e, its utility varies among users, the goal is to find a subset of features that maximizes the total utility over all users. For instance, given a dataset of health attributes of patients, different users may want to predict the likelihood of different diseases, and a set of features may be useful for some illnesses but extraneous for the others. A common approach to measure the utility of set of features is the mutual information between the patients' health attributes restricted to the chosen set of features and the empirical likelihood of the diseases of interest under the Naive Bayes assumption. The function is known to exhibit diminishing returns property and is \emph{monotone submodular} \citep{krause2012near}. 




\begin{definition}[Submodular functions]
    A function $f : 2^{V} \to \mathbb{R}$ is called submodular if it has the diminishing returns property, i.e., for every pair of sets $S \subset T \subset V$ and element $e$ not in $T$, the marginal gain of $e$ over $T$, is no more than the marginal gain of $e$ over $S$; i.e., $f(\{e\} \cup T) - f(T) \leq f(\{e\} \cup S) - f(S)$. For all $S\subset V$, we say that $f(S)$ is the \emph{utility} achieved by set $S$, and we let $\OPT$ denote any arbitrary fixed optimal solution for the problem of maximizing $f$ subject to the constraint that the set of elements picked is of size at most $k$.
\end{definition}

The problem of maximizing a submodular function subject to constraints on the subsets of $V$ one is allowed to pick occurs across domains like computer science, electrical engineering \citep{narayanan1997submodular}, and economics \citep{DBLP:conf/soda/DobzinskiS06}. More specifically in theoretical computer science, submodular maximization is a fundamental instance of \emph{combinatorial optimization}, generalizing problems like Max-Cut and Facility Location \citep{schrijver2003combinatorial}. On the other hand, submodular maximization has also been applied concretely to numerous problems in machine learning such as feature selection for classification \citep{DBLP:conf/uai/KrauseG05}, influence maximization in social networks \citep{DBLP:conf/kdd/KempeKT03}, document and corpus summarization \citep{DBLP:conf/acl/LinB11}, result diversification in recommender systems \citep{DBLP:conf/recsys/ParambathUG16} and exemplar based clustering \citep{DBLP:conf/iccv/DueckF07} - we refer the reader to \cite{DBLP:conf/icml/MitrovicB0K17} and \cite{DBLP:journals/corr/abs-2005-14717} for a more complete list of references. 

In many applications, particularly in machine learning, one must publicly release a solution to the submodular maximization problem, but the information used to perform the optimization is private. For instance, consider the example above where the sensitive health information of patients is used to compute the utility of different sets of features. A good solution may be very sensitive to the information used to compute it, and  consequently may also reveal too much information about the private data set. This motivates the problem of \emph{differentially private} submodular maximization, as first considered in \cite{DBLP:conf/soda/GuptaLMRT10}.

Differential privacy (DP) \citep{DBLP:conf/tcc/DworkMNS06} is an attribute that when satisfied by an algorithm, allows one to promise strong information-theoretic guarantees limiting how much an adversary might learn about the private data set based on the output of that algorithm run on the data set.

\begin{definition}[Differential privacy, \cite{DBLP:conf/tcc/DworkMNS06}]
    Let $\mathcal{X}$ be the set of all possible data records. We say that a pair of data sets $A, B$ drawn from $\mathcal{X}$ are \emph{neighbouring} (denoted $A \sim B$) if they differ in at most one data record $I$, so for instance $A = B \cup \{I\}$. We say that an algorithm $\mathcal{A}$ mapping from data sets derived from $\mathcal{X}$ to some output co-domain $\mathcal{Y}$ is $(\varepsilon, \delta)$-differentially private for some $\varepsilon, \delta >0$ if for all pairs of neighbouring data sets $A, B$ and all $Y \subset \mathcal{Y}$,
    \begin{align*}
        P(\mathcal{A}(A) \in Y) \leq e^{\varepsilon} P(\mathcal{A}(B) \in Y) + \delta.
    \end{align*}
\end{definition}
Put more simply, the above definition says that the likelihood of any set of outcomes $Y$ can vary by at most an $e^\varepsilon$ multiplicative factor and an additive $\delta$ term if we were to add or drop any one record from our data set. Requiring that an algorithm be differentially private ensures that the amount of information revealed by the output is bounded by a function of the \emph{privacy parameters} $\varepsilon$ and $\delta$. In practice, one picks a value of $\varepsilon$ to be a small constant, for instance in the range $(0.1,1)$, although different applications can deviate significantly from this. It is typically required that $\delta = o(1/n)$ to avoid the pathological case of completely revealing one individual's data and claiming that privacy is achieved.

For a given choice of privacy parameters, DP is typically achieved by adding appropriately scaled noise to obfuscate sensitive values in the course of the computation, and may result in trading off some of accuracy to achieve the desired level of privacy.  Such trade-offs have been shown to be intrinsic to achieving differential privacy for many problems, but a careful accounting for the noise added and the privacy gained can let one significantly improve this trade-off.

Although there is an extensive line of work in privately maximizing submodular functions \citep{DBLP:conf/soda/GuptaLMRT10, DBLP:conf/icml/MitrovicB0K17, DBLP:journals/corr/abs-2005-14717,DBLP:journals/corr/abs-2006-15744, DBLP:conf/aistats/SadeghiF21}, as far as we know there is no work on doing so in the \emph{streaming setting}. Submodular maximization under streaming has been studied extensively \citep{DBLP:conf/icml/GomesK10, DBLP:conf/spaa/KumarMVV13, DBLP:conf/kdd/BadanidiyuruMKK14} and must be applied when the data stream is very large and one has limited disk space, or when one has to pay a cost to retain items for future consideration. Most notably, a $(1/2-\theta)$-approximation algorithm,  that retains only $O(\frac{k \log k}{\theta})$ many elements, was introduced by \cite{DBLP:conf/kdd/BadanidiyuruMKK14} for the problem of streaming submodular maximization; this is near-optimal as it is known that one cannot do better than an approximation factor of $1/2$ \citep{DBLP:journals/corr/abs-2003-13459}. 

\textbf{Problem statement:} In this work, we consider the problem at the intersection of these two lines of inquiry i.e. submodular maximization in the streaming setting under the constraint of differential privacy. For every possible private data set $A$ there is a corresponding monotonic (non-decreasing) submodular function $f_A$, a public stream of data $V$, and a cardinality constraint $k$, and we want to find a subset $S$ of $V$ with cardinality at most $k$ that achieves utility close to $f_A(\OPT)$ in the streaming setting. Following previous work \citep{DBLP:conf/soda/GuptaLMRT10, DBLP:conf/icml/MitrovicB0K17}, we assume a bound on the \emph{sensitivity} of $f_A$ to $A$, i.e. for any $A\sim B$, $\max_{S \subset V} |f_A (S) - f_B(S)| \leq 1$. Such a bound on the sensitivity of $f$ is necessary to calibrate the obfuscating noise that is bounded in the course of the computation. Indeed, even in the case where the sensitivity is at most $1$, our lower bound below demonstrates that we are bound to incur non-trivial additive error; it follows that in the case of unbounded sensitivity there can be no $(\epsilon,\delta)$-differentially private algorithm with a non-trivial utility guarantee.

We first develop an algorithm for privately maximizing general submodular functions with bounded sensitivity by adapting the $1/2-\theta$-approximation algorithm of \cite{DBLP:conf/kdd/BadanidiyuruMKK14} along the lines of the \emph{sparse vector technique} \citep{DBLP:conf/stoc/DworkNRRV09}.

\begin{restatable}{theorem}{PSSMLap}
    \label{thm:PSSMLap}
    Given query access to a monotone submodular function $f_A$ with sensitivity $1$ and an input stream $V$ of length $n$, \cref{alg:PSSM} instantiated with Laplace noise, a cardinality constraint of $k$, an approximation parameter $\theta$, a failure probability $\eta$, and privacy parameters $\varepsilon<1$ and $\delta$, is $(\varepsilon,\delta)$-DP and achieves utility at least
    \begin{align*}
        \left(\frac{1 }{2}-\theta \right) f_A (\OPT) - O \left( \frac{k^{1.5} \log^{1.5} \frac{nk}{\delta\eta \theta}}{\varepsilon \sqrt{\theta}} \right).
    \end{align*}
    with probability $1-\eta$, where $\OPT$ is any arbitrary optimal solution for the non-private submodular maximization problem. Further, this algorithm has essentially the same space complexity as the non-private algorithm, retaining $O(\frac{k \log k}{\theta})$ many elements, and it operates in just one pass.
\end{restatable}

\cite{DBLP:conf/soda/GuptaLMRT10} showed that when the private submodular function to be maximized is \emph{decomposable}, i.e., it can be written as a sum of submodular functions of bounded magnitude each corresponding to the data of one individual, then a much smaller amount of noise needs to be added compared to the general setting. This decomposable structure is what characterizes the \emph{Combinatorial Public Projects Problem} introduced by \cite{DBLP:conf/focs/PapadimitriouSS08}, and captures public welfare maximization.

$$\max_{S\subset V,|S|\le k} \sum_{p\in A}f_{\{p\}}(S)$$

For each agent $p$, the function $f_{p}$ is a monotone submodular function whose value is bounded in the range $[0,1]$ (automatically implying a sensitivity of $1$), and we assume that there are a total of $\bm{m}$ agents, i.e., $|A| = m$. The assumption on $f_{p}$ is standard and represents a wide variety of applications. 
More generally, when the constraint is that the submodular summands $f_{p} \in [0,\lambda]$ for some known $\lambda$, we simply apply our algorithm to the function $f_A/\lambda$ and in our final guarantee we incur an additional $\lambda$ factor in the additive error. Our second result shows that similar to the non-streaming setting, it is possible to reduce the additive error for decomposable objectives, but doing so requires significant changes to the analysis.


\begin{restatable}{theorem}{PSSMGumb}
    \label{thm:PSSMGumb}
    Given query access to a decomposable monotone submodular function $f_A$ with $m$ summands over a stream $V$ of length $n$, \cref{alg:PSSM} instantiated with Gumbel noise, an approximation parameter $\theta$, a failure probability $\eta$, and privacy parameters $\varepsilon<1$ and $\delta$, achieves utility at least
    \begin{align*}
        \left(\frac{1 }{2}-\theta \right) f_A(\OPT) - O\left(\frac{k \log^{1.5} \frac{\log \frac{1}{\theta \delta}}{\eta \delta \theta} \log \frac{nk}{\eta \theta} }{\varepsilon \sqrt{\theta}}\right)
    \end{align*}
    with probability $1-\eta$. This algorithm has essentially the same space complexity as the non-private algorithm, retaining $O(\frac{k \log k}{\theta})$ many elements and operates in just one pass.
\end{restatable}

The multiplicative approximation factor equals the $1/2-\theta$-approximation factor of the non-private setting. However, similar to the private non-streaming setting in \cite{DBLP:journals/corr/abs-2005-14717}, we see a trade-off between the multiplicative and additive terms which we can tune via the multiplicative approximation parameter $\theta$. The dependence of the additive error term for this theorem has the optimal dependence on $k/\varepsilon$ (up to logarithmic terms), a fact which we prove by extending previous lower bounds by \cite{DBLP:conf/soda/GuptaLMRT10} for private submodular maximization from the $(\varepsilon, 0)$ to $(\varepsilon, \delta)$-setting. The proof proceeds similarly to that of the lower bound of \cite{DBLP:journals/corr/abs-2009-01220} for $k$-means clustering in the $(\varepsilon, \delta)$ setting, and our formal statement is as follows.

\begin{restatable}{theorem}{PSSMLB}
    \label{thm:PSSMLB}
    For all $0\leq \varepsilon, \delta \leq 1$, $k\in \mathbb{N}$, $n\geq k \frac{e^{\varepsilon}-1}{\delta}$, and $c\geq \frac{4\delta}{e^{\varepsilon}-1}$, if a $(\varepsilon,\delta)$-DP algorithm for the submodular maximization problem for decomposable objectives achieves a multiplicative approximation factor of $c$, it must incur additive error $\Omega((kc/\varepsilon) \log (\varepsilon/\delta))$.    
\end{restatable}

We introduce \cref{alg:PSSM} and prove \cref{thm:PSSMLap} in \cref{sec:Lap}, \cref{thm:PSSMGumb} in \cref{sec:Gumb}, and the lower bound \cref{thm:PSSMLB} in \cref{sec:LB}. We conclude with some experiments in \cref{sec:exp} demonstrating how these methods perform compared to the non-private case and trivial private baselines.

\subsection{Related work}

The work \cite{DBLP:conf/focs/PapadimitriouSS08} introduced the Combinatorial Public Projects problem (CPPP) - given a set $A$ of $m$ agents and $n$ resources, and a private submodular utility function $f_{\{p\}}$ for every agent $p$ the solver must coordinate with the agents to maximize $\sum_{i \in [m]} f_{\{p\}}$, i.e. the sum of their utilities. This problem is interesting because in this setting agents are incentivized to lie to the solver and over-represent the utility they may individually derive from the resources that are picked for the group. The authors showed that unless $NP \subset BPP$, there is no truthful and computationally efficient algorithm for the solver to achieve an approximation ratio better than $n^{1/2 - \epsilon}$ for any $\epsilon > 0$.

The work \cite{DBLP:conf/soda/GuptaLMRT10} was the first to consider the problem of differentially private submodular maximization. They showed that it is possible to privately optimize the objective in CPPP while losing privacy that is \emph{independent} of the number of items picked, and achieved essentially optimal additive error. Since $(\varepsilon, \delta)$ privacy implies approximate $2\varepsilon + \delta$-truthfulness, their result showed that a slight relaxation of the truthfulness condition allowed constant factor approximation if the optimal utility was not too low. They also showed that optimizing a submodular function $\varepsilon$-privately under a cardinality constraint of $k$ over a ground set of $n$ elements must suffer additive loss $\Omega(k \log n/k)$.

The work \cite{DBLP:conf/icml/MitrovicB0K17} considered the more general case of private submodular maximization when the objective function has bounded sensitivity. They were motivated by the problem of \emph{data summarization} under the constraint of differential privacy. They proposed algorithms for both general monotone and non-monotone objectives with bounded sensitivity, and provided extensions to matroid and p-extendable systems constraints. Although the error guarantees of their results match those of \cite{DBLP:conf/soda/GuptaLMRT10} in the case of decomposable functions, for the case of general monotone and non-monotone objectives they get higher additive error.

The work \cite{DBLP:journals/corr/abs-2005-14717} extended the results of \cite{DBLP:conf/soda/GuptaLMRT10} from cardinality to matroid constraints, and from monotone to the non-monotone setting. They achieved this by adapting the Continuous Greedy algorithm of \cite{DBLP:journals/siamcomp/CalinescuCPV11} and the Measured Continuous Greedy algorithm of \cite{DBLP:conf/focs/FeldmanNS11}. They also made a small fix to the privacy analysis of \cite{DBLP:conf/soda/GuptaLMRT10} resulting in a weaker bound (by a constant factor) on the privacy loss.

The work \cite{DBLP:journals/corr/abs-2006-15744} also studied the problem of private submodular and $k$-submodular maximization subject to matroid constraints, and achieved the same multiplicative approximation factor as \cite{DBLP:journals/corr/abs-2005-14717} for monotone submodular maximization. In this work privacy and time-efficiency were optimized at the cost of higher additive error.

The work \cite{DBLP:conf/aistats/SadeghiF21} made further progress on private monotone submodular maximization for submodular functions with \emph{total curvature} at most $\kappa \in [0,1]$ by achieving a $1 - \frac{\kappa}{e}$-approximation algorithm and lower query complexity than the previous works.

In the work \cite{DBLP:journals/corr/abs-2010-12816}, the authors considered a variant of this line of work wherein a sequence of private submodular functions defined over a common public ground set are processed, and at every iteration a set of at most $k$ elements from the ground set must be picked before the function is observed. Here the goal is \emph{regret minimization}, and the authors introduced differentially private algorithms that achieve sub-linear regret with respect to a $(1-1/e)$-approximation factor in the full information and bandit settings.
\section{Preliminaries}

We first make a simplifying technical observation.

\begin{remark}
    Following the setup of the set cover problem in \cite{DBLP:conf/soda/GuptaLMRT10}, we assume that there is a publicly known upper bound on the number of agents which we denote by $m$; as the dependence of the error and space on $m$ will be seen to be logarithmic, even a loose upper bound works well. Alternatively, we can allocate a small fraction of our privacy budget to privatize the number of agents $m$ via the Laplace mechanism (\cref{lem:Laplace}).
\end{remark}

We will use the following \emph{basic} and \emph{advanced composition} laws to reason about the privacy loss that occurs when multiple differentially private mechanisms are used in a modular fashion as subroutines. We follow the formulation in \cite{DBLP:journals/fttcs/DworkR14}.

\begin{theorem}[Basic composition, \cite{DBLP:conf/eurocrypt/DworkKMMN06, DBLP:conf/stoc/DworkL09}]
    If $\mathcal{M}_i$ is $(\varepsilon_i,\delta_i)$-differentially private for $i=1,\dots, k$, then the release of the outputs of all $k$ mechanisms is $(\sum_{i=1}^k \varepsilon_i, \sum_{i=1}^k \delta_i)$-differentially private.
\end{theorem}

\begin{theorem}[Advanced composition, \cite{DBLP:conf/focs/DworkRV10}]
    \label{thm:AdvComp}
    For all $\varepsilon, \delta, \delta' \geq 0$, given a set of $(\varepsilon, \delta)$-differentially private mechanisms, if an adversary adaptively chooses $k$ mechanisms to run on a private data set, then the tuple of responses is $(\varepsilon', k\delta + \delta')$-differentially private for $\varepsilon' = \sqrt{2k \ln (1/\delta')}\varepsilon + k\varepsilon (e^{\varepsilon} - 1)$. In particular, for $\varepsilon'<1$, if $\varepsilon = \frac{\varepsilon'}{2\sqrt{2k \ln (1/\delta)}}$, then the net $k$-fold adaptive release is $(\varepsilon, (k+1)\delta)$-differentially private.
\end{theorem}

We will appeal to the following private mechanism used to choose an element from a public set based on a private score with bounded sensitivity.

\begin{lemma}[Exponential Mechanism, {\cite{DBLP:conf/focs/McSherryT07}}]
    \label{lem:ExpMech}
    Let $q_A : V \to \mathbb{R}$ be a private score function for a public domain $V$ that depends on the input data set $A$. Let $\Delta q = \max_{A \sim B} \max_{v \in V} |q_A(v) - q_B (v)|$ be the sensitivity of $q_{\cdot}$. The exponential mechanism  $\mathcal{M}$ samples an element $v \in V$ with probability $\propto \exp(\varepsilon q_A (x)/(2 \Delta q))$. The exponential mechanism is $\varepsilon$-differentially private. Further, for a finite set $V$, we have that with probability $1-\gamma$,
    \begin{align*}
        q_A (v^*) \geq \max_{v \in V} q_A (v) - \frac{2 \Delta q}{\varepsilon} \ln \frac{|V|}{\gamma}.
    \end{align*}  
\end{lemma}

We will appeal to the following primitive for privatizing real values with bounded sensitivity.

\begin{lemma}[Laplace mechanism, \cite{DBLP:conf/tcc/DworkMNS06}]
    \label{lem:Laplace}
    Given a function $f : 2^{\mathcal{X}} \to \mathbb{R}$ such that $\max_{A \sim B}\allowbreak |f(A) - f(B)| \leq \Delta f$, the mapping $A \mapsto f(A) + \alpha$ where $\alpha \sim \Lap (\Delta f/\varepsilon)$ is $\varepsilon$-differentially private. Here $\Lap(\sigma)$ denotes the standard Laplace distribution with scale parameter $\sigma$.
\end{lemma}

We will also draw random values from the Gumbel distribution for improved privacy guarantees. We recall the distribution function for later use.

\begin{definition}[Gumbel distribution]
 The Gumbel distribution is defined on $\mathbb{R}$. When the mean is $\mu$ and the scale parameter is $\gamma$, the distribution is defined by its CDF
 \begin{align*}
     F(x) = \exp (-e^{-(x-\mu)/\gamma}),
 \end{align*}
 or alternatively its density function
 \begin{align*}
     f(x) = \frac{1}{\gamma} \exp(-(x - \mu)/\gamma + e^{-(x-\mu)/\gamma}).
 \end{align*}
\end{definition}

\subsection{Non-private streaming submodular maximization}

\Cref{alg:nonPrivStream} was introduced in \cite{DBLP:conf/kdd/BadanidiyuruMKK14} to maximize a submodular function $f$ subject to a cardinality constraint of $k$ over a stream $V$ of length $n$.

\begin{algorithm}[ht]
\begin{algorithmic}[0]
\Require Monotonic submodular function $f$, cardinality constraint parameter $k$, element stream $V$, approximation parameter $\theta$, estimate $O$ of the optimum cost $f(\OPT)$
\State $S \gets \emptyset$
\For{$e\in V$}
  \If{$f(e|S) \ge O/(2k)$ and $|S|< k$}
    \State $S \gets S \cup \{e\}$
  \EndIf
\EndFor
\end{algorithmic}
\caption{Streaming algorithm for monotone submodular maximization subject to
a cardinality constraint, \cite{DBLP:conf/kdd/BadanidiyuruMKK14} }
\label{alg:nonPrivStream}
\end{algorithm}

\begin{theorem}[\cite{DBLP:conf/kdd/BadanidiyuruMKK14}]\label{thm:nonPrivStream}
The final solution $S$ satisfies $f(S)\ge\min \{O/2,f(\OPT)-O/2 \}$.
\end{theorem}

\begin{proof}
    We see from the pseudocode of \cref{alg:nonPrivStream} that when an element $e_i \in V$ is processed by the stream, if $S_i$ is the set of accepted elements at that point then $e_i$ is retained and added to the solution if and only if $f(e_i|S_i) \geq O/2k$. 
    
    Let $S = \{ e_{i_1}, \dots, e_{i_{|S|}}\}$.
    We consider two cases depending on the size of $S$. If $|S|=k$, then
    \begin{align*}
        f(S) &= \sum_{j=1}^k f(e_{i_j} | S_{{i_j}}) \\
        &\geq k \cdot \frac{O}{2k} \\
        &\geq O/2.
    \end{align*}
    In the second case, if $|S|<k$, then all elements $e_i \in OPT\setminus S$ must have been rejected i.e. $f(e_i|S_i)<O/(2k)$. Let $OPT \backslash S = \{e_{j_1}, \dots, e_{j_t} \}$ (i.e. $\{j_1, \dots, j_t\}$ is some subset of $\{i_1, \dots, i_{|S|}\})$. Since $f$ is monotonic and $S_i \subset S$ for all $e_i \in V$, we have that $f(e_{j_i}|S) \leq f(e_{j_i}|S_{j_i})<O/(2k)$. We can write
    \begin{align*}
        \sum_{e_{j_i} \in OPT \backslash S} f(e_{j_i} | S) &\leq  k \cdot \frac{O}{2k} \\
        &\leq O/2.
    \end{align*}
    On the other hand, we also have that
    \begin{align*}
        \sum_{e_{j_i} \in OPT \backslash S} f(e_{j_i} | S) &\geq \sum_{i=1}^t f(e_{j_i} | S \cup \{e_{j_1}, \dots, e_{j_{i-1}} \})  \\ 
        &= f(S \cup \{e_{j_1}, \dots, e_{j_t} \}) - f(S)\\
        &= f(OPT) - f(S),
    \end{align*}
    where in the above we use that $f(e_{j_i} | S) \geq f(e_{j_i} | S \cup \{e_{j_1}, \dots, e_{j_{i-1}} \})$ by the submodularity of $f$, and then let the sum of marginal gains telescope. From the two inequalities above we get that
    \begin{align*}
        f(S) \geq f(OPT) - O/2.
    \end{align*}
    The desired lower bound now follows by simply taking the min over the two cases.
\end{proof}
We see that if $O = f(OPT)$, then by setting the threshold according to this value we immediately get a $1/2$-approximation algorithm. In the setting where we do not have this information, but have the promise that the optimum value lies in the range $[E,m]$, we can run multiple copies of this algorithm with a set of geometrically scaling guesses 
\begin{align*}
    \mathcal{O} = \{E, (1 + \theta)E, (1+\theta)^2 E, \dots, (1 + \theta)^{\lfloor \log_{1 + \theta} \frac{m}{E} \rfloor} E,  m \}.
\end{align*} 
Essentially, we try geometrically varying guesses for $f(OPT)$ so that we are assured that there is some $O^{\dagger} \in \mathcal{O}$ such that
\begin{align*}
f(OPT) \leq O^{\dagger} \leq (1 + \theta) f(OPT).
\end{align*}
From the guarantee of \cref{thm:nonPrivStream}, we get that if $S^{O^{\dagger}}$ is the output of \cref{alg:nonPrivStream} when run with $O = O^{\dagger}$, then
\begin{align*}
    f(S^{O^{\dagger}}) &\geq  \min \{O^{\dagger}/2,f(OPT)-O^{\dagger}/2 \} \\
    &\geq \min  \{ f(OPT)/2, (1 - \theta) f(OPT)/2 \} \\
    &\geq (1 - \theta) f(OPT)/2.
\end{align*}
We see that since we must maintain and update all the $S^O$ for $O \in \mathcal{O}$, achieving this guarantee requires that we pay a computational and spatial overhead of $\frac{2 \log m/E}{\theta}$ for a $\frac{1-\theta}{2}$-approximation. We also note that this algorithm requires just one pass over the data stream.

\subsection{The sparse vector technique}

Following \cite{DBLP:journals/fttcs/DworkR14}, we describe a differentially private algorithm that allows us to privatize \cref{alg:nonPrivStream} for our application.

In \cref{alg:nonPrivStream}, the private data set is accessed when the scores of the stream elements are computed. A standard way of privatizing these checks of comparing the private score to the public threshold $O/2k$ is to privatize the scores via the Laplace mechanism (\cref{lem:Laplace}); since the objective function has sensitivity $1$, it follows that the sensitivity of these scores is $1$ unit and one needs to add noise sampled from a $\Lap(1/\varepsilon')$ distribution for this check to be $(\varepsilon',0)$-private. Appealing to the advanced composition law we can argue that for a stream of $n$ elements the net privacy loss is approximately $(\tilde{O}(\sqrt{n} \varepsilon'), \delta)$ for some small $\delta$. 

Since the privacy budget is fixed at $(\varepsilon, \delta)$, every call to the Laplace mechanism must be $\sim \varepsilon/\sqrt{n}$ private, which means we are adding noise on the order of $\sqrt{n}/\varepsilon$ to each threshold check; when $n$ is large and $m$ is small this gives us no utility. However, it can be shown that the net privacy loss actually degrades far more slowly, essentially scaling with the number of elements that pass the privatized check, as opposed to the entire length of the stream $n$. One can hence justify adding noise of a much smaller magnitude ($O(\sqrt{k}/\varepsilon)$). This analysis and technique is called the \emph{sparse vector technique} and is attributed to \cite{DBLP:conf/stoc/DworkNRRV09}; we refer the reader to the end of chapter 3 of \cite{DBLP:journals/fttcs/DworkR14} for a more comprehensive discussion.

\begin{algorithm}[ht]
\begin{algorithmic}
\Require Arbitrary (possibly adaptive) stream of sensitivity $1$ queries $f_1, f_2, \dots,$ a threshold $T$, a cutoff point $k$, threshold noise $\mathcal{D}_{\alpha}$, score noise $\mathcal{D}_{\beta}$. Output is a stream of answers $a_1, a_2, \dots$, $a_i \in \{\bot, \top\}$
\State $\alpha \sim \mathcal{D}_{\alpha}$
\State Let $\hat{T}_i = T + \alpha$ for $i \in \{0, \dots, k-1\}$
\State Let $\ct = 0$
\For{query $e_i$}
    \State Let $\beta_i \sim \mathcal{D}_{\beta}$
    \If{$f_i + \beta_i \geq \hat{T}_{\ct}$}
        \State Output $a_i = \top$
        \State $\ct = \ct + 1$
    \Else
        \State $a_i = \bot$
    \EndIf
    \If{$\ct \geq k$}
        \State Halt
    \EndIf
\EndFor
\end{algorithmic}
\caption{Sparse \cite{DBLP:conf/stoc/DworkNRRV09, DBLP:journals/fttcs/DworkR14}}
\label{alg:sparse}
\end{algorithm}

\begin{theorem}[\cite{DBLP:conf/stoc/DworkNRRV09}]
    \label{thm:sparse}
    For $\sigma = (\sqrt{32 k \ln \frac{1}{\delta}})/\varepsilon$, $\mathcal{D}_\alpha = \mbox{Lap}(\sigma)$ and $\mathcal{D}_\beta = \mbox{Lap} (2\sigma)$, \cref{alg:sparse} is $(\varepsilon, \delta)$-DP. Further, with probability $1-\beta$, for all queries $i$, if $\hat{\nu}_i$ is the privatized query value $f_i + \nu_i$, then
    \begin{align*}
        | \hat{\nu}_i - \nu_i | \leq \frac{(\ln n + \ln \frac{4c}{\beta}) \sqrt{k \ln \frac{2}{\delta} (\sqrt{512} + 1)}}{\varepsilon}.
    \end{align*}
\end{theorem}
\section{Private streaming with Laplace noise}
\label{sec:Lap}

In this section we introduce \cref{alg:PSSM} that privately maximizes submodular functions in both the general and decomposable case with differing privatizing noise distributions. We then prove \cref{thm:PSSMLap} formalizing the private and utility guarantee for the general case; we study the case of decomposable functions in \cref{sec:Gumb}.

\PSSMLap*

We observe that by applying the technique of running \cref{alg:sparse} in parallel for geometrically varying guesses for the optimal utility $f(\OPT)$ as in \cite{DBLP:conf/kdd/BadanidiyuruMKK14}, and then applying the exponential mechanism, we get an algorithm (\cref{alg:PSSM} with Laplace noise) and an immediate utility guarantee for the problem being studied in this paper.

\begin{algorithm}
\algloopdefx{Dip}[1]{\textbf{Do in parallel for} #1}
\begin{algorithmic}
\Require Monotone submodular function $f$, cardinality constraint parameter $k$, failure probability $\beta$, privacy parameters $(\varepsilon, \delta)$, element stream $V$, approximation parameter $\theta$
\If{$f$ has sensitivity $1$ but is not $1$-decomposable}
    \State Let $\varepsilon' = \frac{\varepsilon}{4\sqrt{2 T \ln ((T+1)/\delta)}}$
    \State Let $\sigma = \frac{\sqrt{32 k \ln \frac{T+1}{\delta}}}{\varepsilon'}$
    \State Let $\mathcal{D}_{\alpha} = \Lap (\sigma)$
    \State Let $\mathcal{D}_{\beta} = \Lap (2\sigma)$
\ElsIf{$f$ is $1$-decomposable}
    \State $\gamma = O\left( \frac{\varepsilon}{\sqrt{T} \log^{1.5} T/\delta} \right)$ (exact expression in the proof of \cref{thm:PSSMGumb})
    \State $\mathcal{D}_{\alpha} = \Gumb (\gamma)$
    \State $\mathcal{D}_{\beta} = \Gumb (\gamma)$
\EndIf
\State Let $E = \min \left\{ \frac{k \log n}{\varepsilon}, m/2 \right\}$
\State Let $T = \lceil \log_{1 + \theta} \frac{m}{E}\rceil$ + 1
\State Let $\mathcal{O} =  \{E, (1 + \theta)E, (1+\theta)^2 E, \dots, (1 + \theta)^{\lfloor \log_{1 + \theta} \frac{m}{E} \rfloor} E,  m \}$
\For{$e_i \in V$}
\Dip{$O \in \mathcal{O}$}
    \State $S^O \leftarrow$ \Cref{alg:sparse} $(\mbox{Query stream }(f_A (e|S^O))_{e \in V}, \mbox{threshold }\frac{O}{2k},\mbox{cutoff } k, \mathcal{D}_{\alpha} , \mathcal{D}_{\beta} )$
    \State where
    \begin{align*}
        a_i = \top &\Rightarrow \mbox{ add element }e_i \mbox{ to }S^O \\
        a_i = \bot &\Rightarrow \mbox{ reject element }e_i
    \end{align*}
\EndFor
\State $S^{O^*} \leftarrow EM(\{S^O : O \in \mathcal{O}\}, \varepsilon/2 )$ \\
\Return $S^{O^*}$
\end{algorithmic}
\caption{Private streaming submodular maximization}
\label{alg:PSSM}
\end{algorithm}

We first prove a technical lemma that captures the utility provided by \cref{alg:sparse} when the private utility cheks are performed with some bounded random variables.

\begin{lemma}
    \label{lem:utility}
    If $\alpha \in (a_l, a_u)$ and for all elements $e_i \in V$, $\beta_i \in (b_l, b_u)$, then \cref{alg:sparse} has the promise that when run with threshold $T = O/2k$, if we add all elements $e_i$ for which $a_i = \top$ to $S^O$, then
    \begin{align*}
        f(S^O) \geq \frac{1}{2} \min \{ O, f(\OPT) - O \} - k b_u + k a_l.
    \end{align*}
\end{lemma}

\begin{proof}
    Let $S^O = \{e_{i_1}, \dots, e_{i_k}\}$. Let $S_{i_u}^O := \{e_{i_1}, \dots, e_{i_{u-1}} \}$ for $u \leq k$. Since the selected elements must have succeeded in the threshold check, it must be the case that
    \begin{align*}
        f(e_{i_u}|S_{i_u}^O) + \beta_{i_u} &\geq \frac{O}{2k} + \alpha_{i_u} \\
        \Rightarrow f(e_{i_u}|S_{i_u}^O) + b_u &\geq \frac{O}{2k} + a_l \\
        \Rightarrow f(e_{i_u} | S_{i_u}^O) &\geq \frac{O}{2k} - b_u + a_l
    \end{align*}
    Hence we have that
    \begin{align*}
        f(S) &= \sum_{u=0}^{k} f(e_{i_u} | S_{i_u}^O) \\
        &\geq \frac{O}{2} - k b_u + k a_l
    \end{align*}
    On the other hand if $|S^O| = r < k$, as in the proof of \cref{thm:nonPrivStream}, let $\OPT \backslash S = \{ e_{j_1}, \dots, e_{j_t} \}$, and as before, we have that
    \begin{align*}
        f(S^O) &\geq f(\OPT) - \sum_{e_{j_i} \in \OPT \backslash S^O} f(e_{j_i} | S^O) \\
        &\geq f(\OPT) - \sum_{i=1}^t \left(O/2k +  \alpha_{j_i} - \beta_{{j_i}} \right) \\
        &\geq f(\OPT) - \frac{O}{2} - k b_u + k a_l.
    \end{align*}
\end{proof}

\begin{remark}
    Note that although the noise random variables $\beta_{e_i}$ for $e_i \in O$ are all drawn independently and have mean $0$, we have implicitly conditioned on $S^O$ of elements having already been picked, and so we cannot claim and exploit independence so as to derive a better concentration bound. One would expect to see noise values biased high, making it more likely that that element have passed the check. Although we should be able to derive a concentration bound for the threshold noise random variables $\alpha_O$ that scales as $\tilde{O}(\sqrt{k})$, as the other noise term dominates in magnitude this does not help asymptotically.
\end{remark}

\begin{proof}[Proof of \cref{thm:PSSMLap}]
    From \cref{thm:sparse} and the choice of $\sigma$ in \cref{alg:sparse} we see that each one of the $T = |\mathcal{O}|$ calls to  is $(\varepsilon',\frac{\delta}{T+1})$-DP. Then, by advanced composition (\cref{thm:AdvComp}), it follows that since $\varepsilon'$ as set in the pseudocode is 
    \begin{align*}
        \frac{\varepsilon }{4 \sqrt{2T \ln (T+1)/\delta}},
    \end{align*}
    the $T$ calls to \cref{alg:sparse} are collectively $\left(\varepsilon/2, T\frac{\delta}{T+1} + \frac{\delta}{T+1} \right)$-differentially private. We apply basic composition to account for the $\varepsilon/2$-private call to the exponential mechanism and get $(\varepsilon, \delta)$-differential privacy in sum.
    
    We now derive the the utility guarantee. For all $O\in \mathcal{O}$, and all $\alpha_i$, with probability $1-\frac{\eta}{2kT}$
    \begin{align*}
        |\alpha_i| &\leq \frac{8\log 2T/\eta}{\sigma}.
    \end{align*}
    Similarly, with probability $1-\frac{\eta}{2nT}$, we have that for an element $e_{j}$,
    \begin{align*}
        |\beta_j| &\leq \frac{4\log 2nT/\eta}{\sigma}.
    \end{align*}
    Applying the union bound, it follows that with probability $1-\eta$, for all $n$ elements across all $T$ thresholds as well as for all $T$ guesses the respective noise variables $\alpha_i,\beta_j$ are bounded as above. It follows that we can apply \cref{lem:utility} with $a_u = - a_l = \frac{8\log T/\eta}{\sigma}$, and $b_u = - b_l =  \frac{4\log n/\eta}{\sigma}$. Substituting, we get that for all guesses $O$ with probability $1-\eta/2$,
    \begin{align*}
        f(S^O) \geq \frac{1}{2} \min \left\{ O , 2f(\OPT) - O \right\} - \frac{4k\log 2nT/\eta}{\sigma} - \frac{8k\log 2kT/\eta}{\sigma}
    \end{align*}
    As $O$ varies geometrically between $E$ and $m$, and we have the promise that $f$ takes values in $[E,m]$, it follows that there is a choice $O^{\dagger}$ such that
    \begin{align*}
        f(\OPT) &\leq O^{\dagger} \leq (1 + \theta) f(\OPT) \\
        \Rightarrow \min\{ O^{\dagger}, 2f(\OPT) - O^{\dagger} \} &\geq (1 - \theta) f(\OPT) \\
        \Rightarrow f(S^{O^{\dagger}}) &\geq \frac{(1 - \theta) f(\OPT)}{2} - \frac{12k\log 2nkT/\eta}{\sigma}.
    \end{align*}
    From the guarantee of the exponential mechanism we get that with probability $1-\eta/2$,
    \begin{align*}
        f(S^{O^*}) &\geq f(S^{O^{\dagger}}) - \frac{2}{\varepsilon} \ln \frac{2T}{\eta}.
    \end{align*}
    In sum, putting everything together and applying the union bound, we have that with probability $1-\eta$,
    \begin{align*}
        f(S^{O^*}) &\geq \frac{(1 - \theta) f(\OPT)}{2} - \frac{12k\log (2nkT/\eta) \sqrt{32 k \ln \frac{1}{\delta}}}{\varepsilon'} - \frac{2}{\varepsilon} \ln \frac{2T}{\eta} \\
        &\geq \frac{(1 - \theta) f(\OPT)}{2} - O\left( \frac{k\sqrt{k T \log (T/\delta) \log (1/\delta)} \log nkT/\eta }{\varepsilon} \right) \\
        &\geq \frac{(1 - \theta) f(\OPT)}{2} - O\left( \frac{k^{1.5} \log^{1.5} \frac{n k \log m/E}{\eta \theta \delta} \log^{0.5} m/E}{\varepsilon \sqrt{\theta}} \right).
    \end{align*}
    wherein we use that $T = O(\log_{1 + \theta} m/E) = O\left(\frac{\log m/E }{\theta}\right)$.
    
    We now set $E = \min \left\{ \frac{k\log n}{\varepsilon}, m/2 \right\}$, and show that the claimed bound holds. If $f(\OPT)$ lies in the prescribed interval $[E, m]$, then we have already shown that the claimed bound holds. On the other hand, if $f(\OPT) < E$, the additive loss in utility is
    \begin{align*}
        \frac{k^{1.5} \log^{1.5} \frac{n k \log 2}{\eta \theta \delta} \log^{0.5} 2}{\varepsilon \sqrt{\theta}} &> \frac{k \log n}{\varepsilon}\\
        &> f(\OPT)
    \end{align*}
    and so the RHS of the claimed bound is negative. It follows that any choice of $S^{O^*}$ fulfills the claimed bound trivially, and that we can absorb the $\log m/E < \log 2$ terms in the big-Oh notation to get that unconditionally, with probability $1-\eta$,
    \begin{align*}
        f(S^{O^*}) &\geq \frac{(1 - \theta) f(\OPT)}{2} - O\left( \frac{k^{1.5} \log^{1.5} \frac{n k}{\eta \theta \delta}}{\varepsilon \sqrt{\theta}} \right).
    \end{align*}
\end{proof}
\section{Private streaming with Gumbel Noise}
\label{sec:Gumb}

In this section we prove \cref{thm:PSSMGumb} which formalizes the privacy and utility guarantees of \cref{alg:PSSM} instantiated with Gumbel noise in the case that the submodular utility function to be maximized is decomposable. The guarantee that we formalize in \cref{sec:Lap} for the problem of privately maximizing a submodular function with bounded sensitivity essentially follows from previous work, i.e. applying the Laplace mechanism to privatize the threshold checks in the non-private \cref{alg:nonPrivStream} and using the sparse vector technique to tightly bound the loss in privacy and justify adding manageable amounts of noise to the threshold checks. On the other hand, to derive the guarantee of \cref{thm:PSSMGumb} we must go significantly beyond prior work. In this case we are able to show that the additive error suffered in the uility achieved has the optimal dependence upon the cardinality constraint $k$ and the privacy parameter $\varepsilon$.

In the non-streaming setting, it was shown by \cite{DBLP:conf/soda/GuptaLMRT10} and \cite{DBLP:journals/corr/abs-2005-14717} that for decomposable objectives, when the exponential mechanism is used to pick elements in a greedy manner the net privacy loss is essentially independent of the number of elements picked. Although the exact analysis is complicated, at a high level the key idea is the following. We want to bound the the increase in likelihood of a sequence of elements being picked when we move from one private data set to a neighboring one, say by adding one additional agent. We recall that this means that the utility function to be maximized now has an additional summand corresponding to that agent. This implies that the scores used to pick elements and consequently the probabilities with which they are picked now have a \emph{bias} introduced by the new agent. 

Prior work essentially showed that when the elements are picked greedily according to their marginal scores for the utility function, and the exponential mechanism is used to privatize these picks, the biases introduced over the many picks telescope, and the increase in likelihood is in fact bounded by a small multiplicative term independent of the number of picks made. The fact that these biases telescope can be traced to the fact that the net total utility of the additional agent is 1, and so its net influence over successive rounds is also bounded.

In the streaming setting there is no direct way to adapt the exponential mechanism in this manner. However, inspired by the \emph{Gumbel-max} trick \citep{abernethy2016perturbation}, we turn to the Gumbel distribution to privatize the threshold checks made in \cref{alg:sparse}. The Gumbel max trick says that given scores $q(v)$ for elements $v \in V$ for some finite set $V$, if $Z_v$ are values drawn i.i.d. from a standard Gumbel distribution then
\begin{align*}
	\Pr\left[ v^* + Z_{v^*} = \max_v \{ v + Z_v \} \right] = \frac{\exp (v^*)}{\sum_{v\in V} \exp(v)}.
\end{align*}

In other words, if we were to add random values to the scores of these elements according to the standard Gumbel distribution and pick the element with the maximum noisy score, the distribution over the set of elements $V$ will be identical to that of the exponential mechanism for an appropriate choice of privacy parameter, depending on the sensitivity of the scores. Although this observation does not directly give us anything in the streaming setting, it turns out that by privatizing the threshold checks in \cref{alg:sparse} via the Gumbel distribution instead of the Laplace distribution we can show that the net privacy loss has a similar telescoping behaviour as that of the exponential mechanism in the non-streaming case, and we are again able to recover a bound on the privacy loss that is remarkably \emph{independent} of the number of elements picked.

The key technical lemma in the privacy analysis of \cref{alg:PSSM} with Gumbel noise is the following.

\begin{lemma}
    \label{lem:guptaStyle}
    Algorithm~\ref{alg:sparse} is $(\varepsilon, \delta)$-differentially private for $\mathcal{D}_{\alpha}, \mathcal{D}_{\beta} = \mbox{Gumb} (\gamma)$, where $\gamma = \frac{8}{\varepsilon \ln 2} \log \frac{2}{\varepsilon \delta}$ and $\varepsilon<1$.
\end{lemma}

The proof of this result is technically involved, and we defer the proof to the end of this section. In addition to \cref{lem:guptaStyle}, to derive the utility guarantee we will also need some standard concentration bounds for the Gumbel distribution.

\begin{lemma}
    \label{lem:noiseBound}
    If $\alpha_i \sim Gumbel (\mu, \gamma)$, the following statements hold.
    \begin{enumerate}
        \item With probability $1-\beta$, $x \leq \mu + \gamma \log 1/\beta$.
        \item With probability $1-\beta$, $x \geq \mu - \gamma \log \log \frac{1}{\beta}$.
    \end{enumerate}
\end{lemma}
\begin{proof}
    \begin{enumerate}
        \item We recall that the CDF for $Gumbel(\mu,\gamma)$ is $\exp(-\exp(-(x-\mu)/\gamma))$. Then,
    \begin{align*}
        1 - \exp(-\exp(-(x-\mu)/\gamma)) &\leq \beta \\
        \Leftrightarrow - \exp(-(x - \mu)/\gamma) &\geq \log 1-\beta \\
        \Leftrightarrow \frac{x - \mu}{\gamma} &\geq \log \frac{1}{\log \frac{1}{1-\beta}} \\
        \Leftrightarrow x &\geq \mu + \gamma \log \frac{1}{\log \frac{1}{1-\beta}}.
    \end{align*}
    Using the series expansion $\log 1 - x = -x - x^2/2 - \dots \leq - x$, we have that
    \begin{align*}
        \log \frac{1}{1-\beta} &= - \log 1 - \beta  \\
        &\geq \beta \\
        \Rightarrow \log \frac{1}{\log \frac{1}{1-\beta}} &\leq \log \frac{1}{\beta}
    \end{align*}
    So in particular, if $x \geq \mu + \gamma \log \frac{1}{\beta}$, then $P(\alpha_i \geq x) \leq \beta$.
        \item Similar to the first part, we have that
        \begin{align*}
            \exp(-\exp(-(x-\mu)/\gamma)) &\leq \beta \\
            \Leftrightarrow \exp(-(x-\mu)/\gamma) &\geq \log \frac{1}{\beta} \\
            \Leftrightarrow \frac{(x - \mu)}{\gamma} &\leq - \log \log \frac{1}{\beta} \\
            \Leftrightarrow x &\leq \mu - \gamma \log \log \frac{1}{\beta}.
        \end{align*}
    \end{enumerate}
\end{proof}

With these lemmas we can now derive our main result.

\PSSMGumb*

\begin{proof}
    From \cref{lem:noiseBound}, we have that for every guess $O_i$, with probability $1- \eta/2kT$
    \begin{align*}
        \alpha_i &\in \left[ - \gamma \log \log \frac{2kT}{\eta}, \gamma \log \frac{2kT}{\eta} \right].
    \end{align*}
    Similarly, for every element $e_i$ and every threshold $O$, with probability $1- \eta/2nT$
    \begin{align*}
        \beta_i &\in \left[ - \gamma \log \log \frac{2nT}{\eta}, \gamma \log \frac{2nT}{\eta} \right].
    \end{align*}
    Applying the union bound, it follows that in the notation of \cref{lem:utility} we can set
    \begin{align*}
        a_l &= - \gamma \log \log \frac{2kT}{\eta} \\
        a_u &= \gamma \log \frac{2kT}{\eta}  \\
        b_l &= - \gamma \log \log \frac{2nT}{\eta} \\
        b_u &= \gamma \log \frac{2nT}{\eta}.
    \end{align*}
    and conclude that with probability $1-\eta$, for all thresholds $O$,
    \begin{align*}
        f(S^O) &\geq \frac{1}{2} \min \{O, f(\OPT) - O \} - k\gamma \log \frac{2nT}{\eta} - k \gamma \log \log \frac{2kT}{\eta} \\
        &\geq \frac{1}{2} \min \{O, f(\OPT) - O \} - 2 k\gamma \log \frac{2nkT}{\eta}
    \end{align*}
    As in the proof of \cref{thm:PSSMLap}, as long as $f(\OPT) \in [E,m]$ there exists a choice $O^{\dagger}$ for which
    \begin{align*}
        f(S^{O^{\dagger}}) \geq \frac{(1-\theta) f(\OPT)}{2} - 2 k \gamma \log \frac{2nkT}{\eta}.
    \end{align*}
    From the utility guarantee of the exponential mechanism (\cref{lem:ExpMech}) we have that
    \begin{align*}
        f(S^*) &\geq f(S^{O^{\dagger}}) - \frac{2}{\varepsilon} \log \frac{2T}{\eta} \\
        &\geq \frac{(1-\theta) f(\OPT)}{2} - 2 k \gamma \log \frac{2nkT}{\eta} - \frac{2}{\varepsilon} \log \frac{2T}{\eta} \\
        &\geq \frac{(1-\theta) f(\OPT)}{2} - O\left( \frac{k \sqrt{T} \log^{1.5} \frac{T \log T/\delta}{\varepsilon\delta} \log \frac{2nkT}{\eta}}{\varepsilon} \right) - \frac{2}{\varepsilon} \log \frac{2T}{\eta} \\
        &\geq \frac{(1-\theta) f(\OPT)}{2} - O\left( \frac{k \sqrt{\log m/E} \log^{1.5} \frac{\log \frac{m }{E} \log \frac{\log m/E}{\theta \delta}}{\varepsilon\delta \theta} \log \frac{2nkT}{\eta}}{\varepsilon \sqrt{\theta}} \right)
    \end{align*}
    wherein we use that $\gamma = O\left( \frac{\sqrt{T}}{\varepsilon} \log^{1.5} \frac{T \log T/\delta}{\varepsilon \delta} \right)$, and $T = \log_{1 + \theta} m/E = O\left( \frac{\log m/E}{\theta} \right)$. Again, setting $E = \min \left\{ \frac{k \log n}{\varepsilon}, m/2 \right\}$, we get that if $f(\OPT) \in [E,m]$, then a good choice of $O^{\dagger}$ exists and the desired bound follows. On the other hand, if $f < E < m/2$, then since the additive error term is at least $\frac{k \log n}{\varepsilon}$ which dominates $E \geq \frac{(1-\theta) f(\OPT)}{2}$. This implies that the RHS of the claimed bound is negative, and that any choice of $S^{O^*}$ fulfills the claimed bound trivially. Further, we absorb the $\log m/E = \log 2$ terms in the big-Oh notation to simplify the expression and get
    \begin{align*}
        f(S^*) \geq \frac{(1-\theta)f(\OPT)}{2} - O\left( \frac{k \log^{1.5} \frac{\log \log \frac{1}{\theta\delta}}{\varepsilon\delta \theta} \log \frac{2nkT}{\eta}}{\varepsilon \sqrt{\theta}} \right).
    \end{align*}
    
    We now derive the privacy guarantee. Since the $T \leq \frac{2\log \frac{m}{E}}{\theta}$-many instantiations of $\cref{alg:sparse}$ are run with independent random bits, we can use advanced composition to argue that the net privacy loss suffered by releasing the sets $S^O$ is 
    $$(2 \varepsilon' \sqrt{2T \log (1/\delta')}, (T+1)\delta' ),$$ 
    where it suffices to set $\gamma = \frac{8}{\varepsilon' \ln 2} \log \frac{2}{\varepsilon' \delta'}$ by \cref{lem:guptaStyle}. Replacing $\varepsilon'$ by $\frac{\varepsilon}{4\sqrt{2T \log (1/\delta')}}$ in the expression for $\gamma$, it follows that for \begin{align*}
        \gamma &= \frac{32\sqrt{2T \log T/\delta}}{\varepsilon \ln 2} \log \frac{4T\sqrt{2T \log (T/\delta)}}{\varepsilon \delta} \\
        &= O\left( \frac{1}{\varepsilon\sqrt{\theta}} \log^{1.5} \frac{\log \frac{1}{\theta\delta}}{\theta\varepsilon \delta} \right),
    \end{align*}
    \cref{alg:PSSM} with Gumbel noise is $(\varepsilon/2, \delta)$-differentially private. We can then apply the exponential mechanism on this public set of choices for an additional $(\varepsilon/2, 0)$-privacy loss, giving us the claimed expression. 
\end{proof}

To prove \cref{lem:guptaStyle}, we first derive some technical lemmas that characterize the probability of stream elements succeeding in their privatized threshold checks. \Cref{lem:oneEmt} characterizes the probability of an element $e_i$ being picked condition on the set $S$ already having been picked.

\begin{lemma}
    \label{lem:oneEmt}
    Conditioned on \cref{alg:sparse} having picked the set $S_{i}$ of elements with $|S_{i}| = u <k$,
    \begin{align*}
        \Pr [a_i = \top | S_{i}] = \int_{-\infty}^{\infty} 1 - \exp(-w_A (i|S_{i}) e^{\alpha_{u}/\gamma } ) \dd \alpha_{|S_{i}|},
    \end{align*}
    where the expression $w_A (i|S_{i})$ denotes
    \begin{align*}
        \exp(\frac{-1}{\gamma} \left( \frac{O}{2k} - f_A(e_i | S_{i}) \right) ),
    \end{align*}
    and $\alpha_{u} \sim Gumbel(0, \gamma)$. 
\end{lemma}

\begin{proof}
    We recall that for an element $e_i$ to be picked by \cref{alg:PSSM} (which happens if and only if the output $a_i$ of \cref{alg:sparse} on input $e_i$ equals $\top$), it must be the case that if $S_{i}$ is the set of elements picked thus far then $|S_i| = u < k$, and that the privatized marginal utility of $e_i$ given $S_{i}$ has been picked beats the privatized threshold $O/2k + \alpha_u$. We can write
    \begin{align*}
        \Pr [a_i = \top | S_{i}] &= \Pr [f(e_i | S_i) + \beta_i \geq \frac{O}{2k} + \alpha_{|S|}] \\
        &= \int_{-\infty}^{\infty} 1\left[f(e_i | S_i) + \beta_i \geq \frac{O}{2k} + \alpha_{u}\right] \dd \alpha_{u} \dd \beta_i
    \end{align*}
    where $\alpha_{u}, \beta_i \sim Gumbel(0, \gamma)$, and $1[\cdot]$ denotes the indicator of the event in its argument. Since $\alpha_{u}$ and $\beta_i$ are drawn independently of each other, we can factorize their joint density function and write
    \begin{align}
       \Pr [a_i = \top | S_{i}] &= \int_{-\infty}^{\infty} \Pr[\beta_i \geq \frac{O}{2k} + \alpha_{u} - f(e_i | S_i)] \dd \alpha_{u}. \label{lem:oneEmt;eqn:1}
    \end{align}
    Since $\beta_i \sim Gumbel (0,\gamma)$, we have that
    \begin{align*}
        \Pr[\beta_i \geq \frac{O}{2k} + \alpha_{u} - f(e_i | S_i)] &= 1 - \Pr [ \beta_i < \frac{O}{2k} + \alpha_{u} - f(e_i | S_i) ] \\
        &= 1 - \exp(-\exp(\frac{-1}{\gamma} \left( \frac{O}{2k} + \alpha_{u} - f(e_i | S_i) \right) )) \\
        &= 1 - \exp(- w_A(i|S_{i}) e^{\alpha_u /\gamma}).
    \end{align*}
    Substituting for the integrand in \cref{lem:oneEmt;eqn:1}, we get the stated result.
\end{proof}

\begin{lemma}
    \label{lem:fullStr}
    Let $\mathcal{A}(A)$ denote the set of elements indicated by \cref{alg:sparse} to be picked for the decomposable submodular function $f_A$. If $S = (e_{i_1}, e_{i_2}, \dots, e_{i_r})$, then
    \begin{align*}
        \Pr[\mathcal{A} (A) = S] = \prod_{u = 1}^r \frac{w_A (i_u | S_{i_u})}{(1 + \sum_{j \in (i_{u-1}, i_u)} w_A(j|S_{i_u}))(1 + \sum_{j \in (i_{u-1}, i_u]} w_A(j|S_{i_u}))}.
    \end{align*}
    where $S_{i_u} = \{e_{i_1}, e_{i_2}, \dots, e_{i_{u-1}}\}$, i.e. the set of elements already picked when element $e_{i_u}$ is considered.
\end{lemma}

\begin{proof}
    The stream $e_1, e_2, \dots$ is given as input to \cref{alg:sparse} and the elements $e_{i_1}, \dots, e_{i_r}$ are picked. For this to be the case, for every $u \in \{1, \dots, r\}$, all the elements after $i_{u-1}$ (the element picked before $i_u$)), and before $i_u$ must fail their privatized checks conditioned on $\{i_1, \dots, i_{u-1}\}$ having already been picked, but $i_u$ must itself succeed. In a way similar to \cref{lem:oneEmt}, we can factor the joint density function of the $\alpha_{u}$ and the $\beta_{j}$ as they are drawn independently and write
    \begin{align}
        \Pr[\mathcal{A} (A) = S] &= \prod_{u = 1}^r \Pr[ a_{i_u} = \top | S_{i_u} ] \prod_{j \in (i_{u-1}, i_u)} \Pr[ a_j = \bot | S_{i_u} ]  \nonumber \\
        &= \prod_{u = 1}^r \int_{-\infty}^{\infty} (1 - \exp(-w_A(i_u | u )e^{\alpha_{u}/\gamma}))  \prod_{j \in (i_{u-1}, i_u)} \exp(-w_A(j|S_{i_u})e^{\alpha_{u}/\gamma}) \dd \alpha_{u} \nonumber \\
        &= \prod_{u = 1}^r \int_{-\infty}^{\infty} (1 - \exp(-w_A(i_u | u )e^{ z})) \exp( - e^{ z} \sum_{j \in (i_{u-1}, i_u)} w_A(j|S_{i_u})) \exp(-z - e^{-z}) \dd z \nonumber 
    \end{align}
    In the above we use that the PDF of $\alpha_u$ is $P(x) = \frac{1}{\gamma}\exp(-(\frac{x}{\gamma} + e^{-x/\gamma}))$, and make the variable substitution $z = x/\gamma$. We can simplify the integrand of each factor as follows.
    \begin{align*}
        &(1 - \exp(-w_A(i_u | S_{i_u} )e^{ z})) \exp(- e^{ z} \sum_{j \in (i_{u-1}, i_u)} w_A(j|S_{i_u}))\exp(- z - e^{-z}) \\
        &= \exp( -z - e^{-z}(1 + \sum_{j \in (i_{u-1}, i_u)} w_A(j|S_{i_u}))) - \exp( -z - e^{-z}(1 + \sum_{j \in (i_{u-1}, i_u]} w_A(j|S_{i_u})))
    \end{align*}
    We can integrate the summands separately; as they have the same form, to derive the resulting expression it suffices to compute the first integral, which we denote $I$.
    \begin{align*}
        I =\int_{-\infty}^{\infty} \exp( -z - e^{-z}(1 + \sum_{j \in (i_{u-1}, i_u)} w_A(j|S_{i_u})) ) \dd z
    \end{align*}
    Let $t =  - e^{-z} (1 + \sum_{j \in (i_{u-1}, i_u)} w_A(j|S_{i_u}))$. Then, $\dd t = e^{-z} (1 + \sum_{j \in (i_{u-1}, i_u)} w_A(j|S_{i_u})) \dd z$. Substituting this variable, we get
    \begin{align*}
        I &= \frac{1}{(1 + \sum_{j \in (i_{u-1}, i_u)} w_A(j|S_{i_u}))} \int_{-\infty}^{0} e^t \dd t\\
        &= \frac{1}{1 + \sum_{j \in (i_{u-1}, i_u)} w_A(j|S_{i_u})}.
    \end{align*}
    It follows that
    \begin{align*}
        \Pr[\mathcal{A} (A) = S] &= \prod_{u = 1}^r \frac{1}{1 + \sum_{j \in (i_{u-1}, i_u)} w_A(j|S_{i_u})} - \frac{1}{1 + \sum_{j \in (i_{u-1}, i_u]} w_A(j|S_{i_u})} \\
        &= \prod_{u = 1}^r \frac{w_A (i_u | S_{i_u})}{(1 + \sum_{j \in (i_{u-1}, i_u)} w_A(j|S_{i_u}))(1 + \sum_{j \in (i_{u-1}, i_u]} w_A(j|S_{i_u}))}.
    \end{align*}
\end{proof}

\begin{proof}[Proof of \cref{lem:guptaStyle}]
    Let $A, B \subset \mathcal{X}$, and let $A = B \sqcup \{I\}$. Let $S = (e_{i_1}, \dots, e_{i_r} )$ be any sequence of elements that can be picked by the algorithm (i.e. $r \leq k$). First we show that
    \begin{align*}
        \Pr [\mathcal{A} (A) = S] \leq e^{\varepsilon'} \Pr [\mathcal{A} (B) = S] + \delta
    \end{align*}
    Substituting from \cref{lem:fullStr} and rearranging terms, we have that
    \begin{align*}
        \frac{\Pr [\mathcal{A} (A) = E]}{\Pr [\mathcal{A} (B) = E]} &= \prod_{u = 1}^r \frac{w_A (i_u | S_{i_u})}{w_B (i_u |S_{i_u})} \cdot \prod_{u = 1}^r \frac{(1 + \sum_{j \in (i_{u-1}, i_u)} w_B(j|S_{i_u}))(1 + \sum_{j \in (i_{u-1}, i_u]} w_B(j|S_{i_u}))}{(1 + \sum_{j \in (i_{u-1}, i_u)} w_A(j|S_{i_u}))(1 + \sum_{j \in (i_{u-1}, i_u]} w_A(j|S_{i_u}))}.
    \end{align*}
    We bound the two factors of this expression separately. For the first factor, we have
    \begin{align*}
        \prod_{u = 1}^r \frac{w_A (i_u | S_{i_u})}{w_B (i_u | S_{i_u})} &= \prod_{u = 1}^r \frac{\exp(\frac{-1}{\gamma} \left( \frac{O}{2k} - f_A (e_{i_{u}} | S_{i_u}) \right) )}{\exp(\frac{-1}{\gamma} \left( \frac{O}{2k} - f_B (e_{i_{u}} | S_{i_u}) \right) )} \\
        &= \exp (\frac{-1}{\gamma} \sum_{u=1}^r f_B (e_{i_{u}}|S_{i_u}) - f_A (e_{i_{u}} | S_{i_u}) ) \\
        &= \exp (\frac{1}{\gamma} \sum_{u=1}^r f_p (e_{i_{u}}|S_{i_u})) \\
        &\leq \exp (1/\gamma).
    \end{align*}
    The second factor is bounded trivially from above by $1$; to see this, we observe that the following sequence of inequalities holds.
    \begin{align*}
        w_A (j|u) &= \exp(\frac{-1}{\gamma} \left( \frac{O}{2k} - f_A(e_i|S_{i_u}) \right) ) \\
        &= \exp(\frac{-1}{\gamma} \left( \frac{O}{2k} - f_B(e_i|S_{i_u}) \right) + \frac{f_p(e_i | S_{i_u}))}{\gamma}) \\
        &\geq \exp(\frac{-1}{\gamma} \left( \frac{O}{2k} - f_B(e_i|S_{i_u}) \right) ) \\
        &\geq w_B (j|u).
    \end{align*}
    In sum, it follows that any value of $\gamma \geq 1/\varepsilon$ suffices. We now show that
    \begin{align*}
        &\Pr [\mathcal{A} (B) = S] \leq e^{\varepsilon} \Pr [\mathcal{A} (B) = S] + \delta.
    \end{align*}
    To this end we consider the reciprocal of the ratio we bounded for the first case, i.e.
    \begin{align}
        \frac{\Pr [\mathcal{A} (B) = E]}{\Pr [\mathcal{A} (A) = E]} &= \prod_{u = 1}^r \frac{w_B (i_u | S_{i_u})}{w_A (i_u | S_{i_u})} \cdot \prod_{u = 1}^r \frac{(1 + \sum_{j \in (i_{u-1}, i_u)} w_A(j|S_{i_u}))(1 + \sum_{j \in (i_{u-1}, i_u]} w_A(j|S_{i_u}))}{(1 + \sum_{j \in (i_{u-1}, i_u)} w_B(j|S_{i_u}))(1 + \sum_{j \in (i_{u-1}, i_u]} w_B(j|S_{i_u}))}. \label{eqn:probRatio}
    \end{align}
    To bound this ratio, we first derive a simple relaxation.
    \begin{claim}
        The following bound holds:
        \begin{align*}
            \frac{\Pr [\mathcal{A} (B) = E]}{\Pr [\mathcal{A} (A) = E]} &\leq \left( \prod_{u=1}^r \frac{1 + \sum_{j \in (i_{u-1}, i_u)} w_A(j|S_{i_u})}{1 + \sum_{j \in (i_{u-1}, i_u)} w_B(j|S_{i_u})} \right)^2.
        \end{align*}
    \end{claim}
    \begin{proof}
        We observe that
        \begin{align*}
            \frac{w_B (i_u | S_{i_u})}{w_A (i_u | S_{i_u})} = \exp(\frac{- f_{\{p\}} (e_{i_u} | S)}{\gamma} ) \leq 1,
        \end{align*}
        and that
        \begin{align*}
            &\frac{w_B (i_u | S_{i_u})}{w_A (i_u | S_{i_u})} \cdot \frac{1 + \sum_{j \in (i_{u-1}, i_u]} w_A(j|S_{i_u})}{1 + \sum_{j \in (i_{u-1}, i_u]} w_B(j|S_{i_u})} \\
            &= \exp(\frac{- f_{\{p\}} (e_{i_u} | S)}{\gamma} ) \cdot \frac{1 + \sum_{j \in (i_{u-1}, i_u)} w_A(j|S_{i_u}) + w_A(i_u | S_{i_u})}{1 + \sum_{j \in (i_{u-1}, i_u)} w_B(j|S_{i_u}) + w_B(i_u | S_{i_u})} \\
            &= \exp(\frac{- f_{\{p\}} (e_{i_u} | S)}{\gamma} ) \cdot \frac{1 + \sum_{j \in (i_{u-1}, i_u)} w_A(j|S_{i_u}) + \exp(f_{\{p\}} (e_{i_u} | S) / \gamma ) \cdot w_B(i_u | S_{i_u})}{1 + \sum_{j \in (i_{u-1}, i_u)} w_B(j|S_{i_u}) + w_B(i_u | S_{i_u})}\\
            &\leq \frac{\exp( - f_{\{p\}} (e_{i_u} | S) / \gamma ) + \exp( - f_{\{p\}} (e_{i_u} | S) / \gamma ) \sum_{j \in (i_{u-1}, i_u)} w_A(j|S_{i_u}) + w_B(i_u | S_{i_u})}{1 + \sum_{j \in (i_{u-1}, i_u)} w_B(j|S_{i_u}) + w_B(i_u | S_{i_u})}\\
            &\leq  \frac{1 + \sum_{j \in (i_{u-1}, i_u)} w_A(j|S_{i_u}) + w_B(i_u | S_{i_u})}{1 + \sum_{j \in (i_{u-1}, i_u)} w_B(j|S_{i_u}) + w_B(i_u | S_{i_u})} \\
            &< \frac{1 + \sum_{j \in (i_{u-1}, i_u)} w_A(j|S_{i_u})}{1 + \sum_{j \in (i_{u-1}, i_u)} w_B(j|S_{i_u})}.
        \end{align*}
        The claim now follows by applying this upper bound for each factor in \cref{eqn:probRatio} as $u$ varies from $1$ to $r$.
    \end{proof}
    
    We will now focus on bounding the expression
    \begin{align*}
        \frac{1 + \sum_{j \in (i_{u-1}, i_u)} w_A(j|S_{i_u})}{1 + \sum_{j \in (i_{u-1}, i_u)} w_B(j|S_{i_u})}.
    \end{align*}
    We first observe that this expression can be identified with the expectation of an monotonically increasing function of the marginal utility of the agent $\{p\} = A \backslash B$.
    \begin{definition}
        Let $v_j :=  w_B ({i_{u-1} + j} | S_{i_u})$. Let $P_u$ be a distribution over $\{\bot\} \cup \{1, \dots, i_{u} - i_{u-1}\}$ such that $P_u (i_{u} + j) \propto v_j$ and $P_u (\bot) \propto 1$. With this definition, we see that
        \begin{align*}
            \frac{1 + \sum_{j >0} w_A(i_u + j|S_{i_u})}{1 + \sum_{j > 0} w_B (i_u + j|S_{i_u})} &= \frac{1 + \sum_{j \in (i_{u-1}, i_u)} \exp(f_{\{p\}} (e_j | S_{i_u})/\gamma) w_B (i_u + j|S_{i_u})}{1 + \sum_{j > 0} w_B(i_u + j|S_{i_u})}   \\
            &= \frac{1 + \sum_{j \in (i_{u-1}, i_u)} \exp (f_{\{p\}} (e_j | S_{i_u})) v_j}{1 + \sum_{j \in (i_{u-1}, i_u)} v_j} \\
            &= \Ex_{j \sim P_u} \left[ \exp(f_{\{p\}} (e_{j}|S)/\gamma) \right]
        \end{align*}
        We let $Y_u := \Ex_{j \sim P_u} [\exp(f_{\{p\}} (e_j | S_{i_u})/\gamma)]$.
    \end{definition}
    At a high level the key insight is that since the sum of marginal utilities $\sum_{u=1}^r f_{\{p\}} (e_{i_u} | S_{i_u})$ of any agent is at most $1$, the net privacy loss, which we have shown to be bounded above by a product of monotonic functions of the sequential marginal utilities may also be bounded more tightly than the $\tilde{O}(\sqrt{k}/\varepsilon)$ bound that arises from advanced composition. To prove this stronger concentration bound we first derive a moment bound on these functions of the expected marginal utility $Y_u$. The complication here is that the elements picked by the algorithm affect the marginal utilities of all subsequent elements considered; we proceed by formalizing a probabilistic process capturing the behaviour of this algorithm in a manner similar to that of \cite{DBLP:conf/soda/GuptaLMRT10} and \cite{DBLP:journals/corr/abs-2005-14717}.
    \begin{lemma}
        \label{lem:guptaStyleTechnical}
        Consider the following $k$-round probabilistic process. We recall that $v_j :=  w_B ({i_{u-1} + j} | S_{i_u})$. In each round $u$, it is the case that the set of elements $S_{i_u} = \{i_1, \dots, i_{u-1} \}$ has been picked, and the element $i_u = j + i_{u-1}$ is picked with probability
        \begin{align*}
            p_j = \frac{1}{1 + v_1 + \dots + v_{j-1}} \cdot \frac{v_j}{1 + v_1 + \dots + v_j }.
        \end{align*}
        Then, for each $q = 1,\dots, r$, for a value of $c = \gamma/4$ and $\varepsilon<1$, the following bound holds:
        \begin{align*}
            \Ex_{S} \left[\prod_{u=q}^k Y_u^c | S_{i_{q}} \right] \leq 1 + \frac{1}{\varepsilon} (1 - f_{\{p\}}(S_{i_{q}})).
        \end{align*}
    \end{lemma}
    
    Before we prove this lemma, we first prove a minor claim that linearizes the dependence on the moments $Y_u^c$ on the marginal utility random variable $f_{\{p\}} (e_{i_u + j} | S_{i_u})$.
    
    \begin{claim}
        For $c, \gamma$ such that $1\leq c < \gamma$,
        \begin{align*}
            Y_u^c &\leq 1 + \frac{(e-1)c}{\gamma} \Ex_{j\sim P_u} [f_{\{p\}} (e_{i_{u-1} + j } | S_{i_u})].
        \end{align*}
    \end{claim}
    \begin{proof}
        By definition, we have that
        \begin{align*}
            Y_u^c &=  \left(\frac{1 + \sum_{j \in (i_{u-1}, i_u)} \exp( f_{\{p\}} (e_{i_{u-1}+j} | S_{i_u})/ \gamma  ) v_j}{1 + \sum_{j \in (i_{u-1}, i_u)} v_j } \right)^c.
        \end{align*}
        Applying Jensen's inequality, we get that
        \begin{align*}
            Y_u^c &\leq  \frac{1 + \sum_{j \in (i_{u-1}, i_u)} \exp( c f_{\{p\}} (e_{i_{u-1}+j} | S_{i_u})/ \gamma  ) v_j}{1 + \sum_{j \in (i_{u-1}, i_u)} v_j}.
        \end{align*}
        Since $c<\gamma$ and $f_{\{p\}} (e_{i_{u-1} + j} | S_{i_u}) \leq 1$, by applying the inequality $e^x < 1 + (e-1)x$ for $x\leq 1$, it follows that 
        \begin{align*}
            \exp(c f_{\{p\}} (e_{i_{u-1} + j} | S_{i_u})/\gamma) \leq 1 + (e-1) c f_{\{p\}} (e_{i_{u-1} + j} | S_{i_u}).
        \end{align*}
        Applying this bound and continuing, we get that
        \begin{align*}
            Y_u^c &\leq \frac{1 + \sum_{j \in (i_{u-1}, i_u)} (1 + (e-1) c f_{\{p\}} (e_{i_{u-1} + j} | S_{i_u})/\gamma )  w_B(i_{u-1} + j|S_{i_u}) }{1 + \sum_{j \in (i_{u-1}, i_u)} w_B(i_{u-1} + j|S_{i_u})} \\ 
            &\leq  1 + \frac{ \sum_{j \in (i_{u-1}, i_u)} (e-1) c f_{\{p\}} (e_{i_{u-1} + j} | S_{i_u}) w_B(i_{u-1} + j|S_{i_u})/\gamma }{1 + \sum_{j \in (i_{u-1}, i_u)} w_B(i_{u-1} + j|S_{i_u})} \\ 
            &\leq 1 + \frac{(e-1)c}{\gamma} \Ex_{j\sim P_u} [f_{\{p\}} (e_{i_u + j } | S_{i_u})].
        \end{align*}
    \end{proof}
    
    \begin{proof}[Proof of \cref{lem:guptaStyleTechnical}]
        We proceed by reverse induction on $q$. For the base case, i.e. $q = k$, we have that
        \begin{align*}
            \Ex_{i_{k}}[ Y_k^c ] &\leq \Ex\left[ 1 + \frac{(e-1)c}{\gamma} \Ex_{j\sim P_k} [f_{\{p\}} (e_{i_{k-1} + j }) ] \right]  \\
            &\leq 1 + \frac{(e-1)c}{\gamma} \sup_{j>0} f_{\{p\}} (e_{i_{k-1} + j} | S_{i_{k}}) \\
            &\leq 1 + \frac{(e-1)c}{\gamma} (1 - f_{\{p\}} (S_{i_{k}})) \\
            &\leq 1 + \frac{1}{\varepsilon} (1 - f_{\{p\}} (S_{i_{k}})),
        \end{align*}
        where in the above we use that $c/\gamma = 1/4 \leq \frac{1}{(e-1)\varepsilon}$ for $\varepsilon \leq 1$, and that $f_{\{p\}} (e_{i_{k-1} + j} | S_{i_{k}}) + f_{\{p\}} (S_{i_{k}}) = f_{\{p\}} (S_{i_{k}} \cup \{ e_{i_{k-1} + j}\})\leq 1$ for any choice of $j$. For the induction step, we assume that the statement is true for $u = q+1,\dots, k$, and derive a bound for the case $u = q$.
        \begin{align*}
            &\Ex_{i_q, \dots, i_k} \left[\prod_{u=q}^k Y_u^c | S_{i_{q}} \right] \\
            &= \Ex_{i_q} \left[ Y_q^c \cdot \Ex_{i_{q+1}, \dots, i_k} \left[\prod_{u=q+1}^k Y_u^c | S_q \right] |S_{i_{q}}\right] \\
            &\leq \Ex_{i_q} \left[ Y_q^c \left( 1 + \frac{1}{\varepsilon} (1 - f_{\{p\}} (S_{q})) \right)  | S_{i_{q}} \right]\\
            &\leq \Ex_{i_q} \left[ Y_q^c \left( 1 + \frac{1}{\varepsilon} (1 - f_{\{p\}} (S_{i_{q}}) - f_{\{p\}} (e_{i_q} | S_{i_{q}})) \right)  | S_{i_{q}} \right]\\
            &\leq \Ex_{i_q} \left[ \left( 1 + \frac{(e-1)c}{\gamma} \Ex_{j \sim P_q}[f_{\{p\}} (e_j | S_{i_{q}})] \right) \left( 1 + \frac{1}{\varepsilon} (1 - f_{\{p\}} (S_{i_{q}}) - f_{\{p\}} (e_{i_q} | S_{i_{q}})) \right)  | S_{i_{q}} \right] \\
            &\leq 1 + \frac{1}{\varepsilon} (1 - f_{\{p\}} (S_{i_{q}})) -  \frac{1}{\varepsilon} \Ex_{i_q}[ f_{\{p\}} (e_{i_q} | S_{i_{q}}) | S_{i_{q}} ] \\
            &+ \frac{(e-1)}{4} \Ex_{i_q}\left[ \Ex_{j \sim P_q} [f_{\{p\}} (e_{j} | S_{i_{q}})]  \cdot \left(1 + \frac{1}{\varepsilon} (1 - f_{\{p\}} (S_{i_{q}}) - f_{\{p\}} (e_{i_q} | S_{i_{q}}) )  \right) \right] \\
            &\leq 1 + \frac{1}{\varepsilon} (1 - f_{\{p\}} (S_{i_{q}})) -  \frac{1}{\varepsilon} \Ex_{i_q}[ f_{\{p\}} (e_{i_q} | S_{i_{q}}) | S_{i_{q}} ] + \frac{(e-1) (1 + 1/\varepsilon)}{4} \Ex_{i_q} [ \Ex_{j \sim P_q} [f_{\{p\}} (e_{j} | S_{i_{q}})] ],
        \end{align*}
        where in the above we use that 
        \begin{align*}
            1 + \frac{1}{\varepsilon} (1 - f_{\{p\}} (S_{i_{q}}) - f_{\{p\}} (e_{i_q} | S_{i_{q}}) )  &\leq 1 + \frac{1}{\varepsilon}.
        \end{align*}
        It follows that it would suffice to show that the last two terms sum to at most $0$. We have that
        \begin{align*}
            \Ex_{i_q}[ \Ex_{j \sim P_q} [f_{\{p\}} (e_{j} | S_{i_{q}})] ] &= \sum_{w \geq 1} p_w \Ex_{j \sim P_q} [f_{\{p\}} (e_{j} | S_{i_{q}}) | i_q = i_{q-1} + w] \\
            &= \sum_{w\geq 1} p_w \sum_{x < w} \frac{v_x}{1 + v_1 + \dots + v_{w-1}} f_{\{p\}} (e_{i_{q-1} + x} | S_{i_{q}})
        \end{align*}
        The outer expectation in the display above corresponds to $i_q$ being picked as described by the probabilistic process (and \cref{alg:PSSM}), and the expectation inside is the expression that we used to bound the privacy loss term for any one round; conditioned on the choice of $i_q$ we recall that it is a distribution over $(i_{q-1}, i_q)$. We switch the sums in the display above to get
        \begin{align*}
            \Ex_{i_q}[ \Ex_{j \sim P_q} [f_{\{p\}} (e_{j} | S_{i_{q}})] ] &= \sum_{x \geq 1} f_{\{p\}} (e_{i_{q-1} + x} | S_{i_{q}}) \sum_{w > x} \frac{v_x}{1 + v_1 + \dots + v_{w-1}} p_w \\
            &\leq \sum_{x \geq 1} f_{\{p\}} (e_{i_{q-1} + x} | S_{i_{q}}) \frac{v_x}{1 + v_1 + \dots + v_{x}} \sum_{w > x}  p_w.
        \end{align*}
        Further we have
        \begin{align*}
            \sum_{w \geq x} p_w &= \sum_{w \geq x} \frac{1}{1 + v_1 + \dots + v_{w-1}} \cdot \frac{v_w}{1 + v_1 + \dots + v_w} \\
            &= \sum_{w \geq x} \frac{1}{1 + v_1 + \dots + v_{w-1}} - \frac{v_w}{1 + v_1 + \dots + v_w} \\
            &= \sum_{w \geq x} \frac{1}{1 + v_1 + \dots + v_{w-1}} - \frac{v_w}{1 + v_1 + \dots + v_w} \\
            &= \frac{1}{1 + v_1 + \dots + v_{x-1}}.
        \end{align*}
        Substituting, we get
        \begin{align*}
            \Ex_{i_q}[ \Ex_{j \sim P_q} [f_{\{p\}} (e_{j} | S_{i_{q}})] ] &\leq \sum_{x \geq 1} f_{\{p\}} (e_{i_{q-1} + x} | S_{i_{q}}) \frac{v_x}{1 + v_1 + \dots + v_x} \cdot \frac{1}{1 + v_1 + \dots + v_{x-1}} \\
            &= \sum_{x \geq 1} p_x f_{\{p\}} (e_{i_{q-1} + x} | S_{i_{q}}) \\
            &= \Ex_{i_q} [f_{\{p\}} (i_q | S_{i_{q}})].
        \end{align*}
        So in sum, we have that
        \begin{align*}
            &\Ex_{i_q, \dots, i_k} \left[\prod_{u=q}^k Y_u^c | S_{i_{q}} \right] \\
            &\leq 1 + \frac{1}{\varepsilon} (1 - f_{\{p\}} (S_{i_{q}})) + \Ex_{i_q} [f_{\{p\}} (i_q | S_{i_{q}})] \left( - \frac{1}{\varepsilon} + \frac{(e-1) (1 + 1/\varepsilon)}{4} \right) \\
            &\leq 1 + \frac{1}{\varepsilon} (1 - f_{\{p\}} (S_{i_{q}})).
        \end{align*}
        wherein we use that for $\varepsilon<1$, $\frac{(e-1)(1 + 1/\varepsilon)}{4} < 1/\varepsilon$.
    \end{proof}
    
    Returning to the proof of \cref{lem:guptaStyle}, we see that the probabilistic process defined and analysed in \cref{lem:guptaStyleTechnical} can be identified with a run of \cref{alg:sparse} with Gumbel noise, where the input stream has been appended with infinitely many items of $0$ marginal utility - this ensures that $k$ complete rounds are executed, but the output distribution on the non-trivial items is identical. Setting $q = 1$ in \cref{lem:guptaStyleTechnical}, since $f_{\{p\}}(\emptyset) = 0$, we see that
    \begin{align*}
        &\Ex_{S} \left[\prod_{u=1}^k Y_u^c | S_{i_{q}} \right] \leq 1 + \frac{1}{\varepsilon} \\
        \Rightarrow &\Ex_{i_1, \dots, i_k}\left[ \left( \prod_{u=1}^k \frac{1 + \sum_{j\in (i_{u-1}, i_u)} w_A (j | S_{i_u}) }{1 + \sum_{j\in (i_{u-1}, i_u)} w_B (j | S_{i_u})} \right)^c \right] \leq 1 + \frac{1}{\varepsilon} \\
        \Rightarrow &\Pr_{i_1,\dots, i_k}\left[ \left( \prod_{u=1}^k \frac{1 + \sum_{j\in (i_{u-1}, i_u)} w_A (j | S_{i_u}) }{1 + \sum_{j\in (i_{u-1}, i_u)} w_B (j | S_{i_u})} \right) > (1 + \varepsilon)^{1/2} \right] \leq \frac{(1 + 1/\varepsilon)}{(1 + \varepsilon)^{c/2}},
    \end{align*}
    wherein in the last step we apply Markov's inequality. Since $\varepsilon<1$, we have that
    \begin{align*}
        \frac{(1 + 1/\varepsilon)}{(1 + \varepsilon)^{c/2}} &\leq \frac{2/\varepsilon}{(1 + \varepsilon)^{c/2}} \\
        &\leq \frac{2/\varepsilon}{\exp( \varepsilon \ln 2 \cdot \frac{c}{2} )} ,
    \end{align*}
    wherein we use that for $\varepsilon < 1$, $1 + \varepsilon \geq \exp(\varepsilon \cdot \ln 2 )$. Setting $c = \frac{2}{\varepsilon \ln 2 } \log \frac{2}{\varepsilon\delta}$, which we note is $\geq 1$, we get that
    \begin{align*}
        \frac{(1 + 1/\varepsilon)}{(1 + \varepsilon)^{c/2}} &\leq \frac{2 / \varepsilon}{\exp(\ln 2/\varepsilon \delta)} \\
        &= \delta.
    \end{align*}
    It follows that with probability $1-\delta$,
    \begin{align*}
        \prod_{u=1}^k \frac{1 + \sum_{j\in (i_{u-1}, i_u)} w_A (j | S_{i_u}) }{1 + \sum_{j\in (i_{u-1}, i_u)} w_B (j | S_{i_u})} &\leq (1 + \varepsilon)^{1/2} \\
        \Rightarrow \frac{\Pr[\mathcal{A}(B) = E]}{\Pr[\mathcal{A}(A) = E]} &\leq 1 + \varepsilon.
    \end{align*}
    It follows that a run of \cref{alg:sparse} with Gumbel noise with noise parameter $\gamma = 4c = \frac{8}{\varepsilon \ln 2} \log \frac{2}{\varepsilon \delta}$ is $(\varepsilon, \delta)$-DP.
    \end{proof}
\section{Lower bound}
\label{sec:LB}

\newcommand{\calS}{\mathcal{S}}
\newcommand{\calT}{\mathcal{T}}

In this section we describe and prove a lower bound (\cref{thm:PSSMLB}) for private submodular maximization. This is a slightly weaker bound than that of \cite{DBLP:conf/soda/GuptaLMRT10} but is more general as it applies to the $(\varepsilon, \delta)$ instead of the $(\varepsilon, 0)$ setting. Further, it also happens to have a decomposable objective, showing that \cref{alg:PSSM} with Gumbel noise has the optimal dependence on $k$ and $\varepsilon$ (up to logarithmic terms).

\begin{definition}[Maximum coverage]
    Given a set system $(U, \calS)$, i.e. a ground set $U$ and a family $\calS$ of subsets of $U$, the maximum coverage problem fixes a private target subset $R\subset U$ and a number $k$ and asks the solver to pick $\calT \subset \calS$ such that $R \subset \cup_{T \in \calT} T$ and $|\calT| \leq k$.
\end{definition}

We can recast this problem in the form of submodular maximization, and then construct a hard instance of maximum coverage to prove our lower bound for $(\varepsilon, \delta)$-DP submodular maximization.

\begin{lemma}[Maximum coverage as submodular maximization]
    Given a set system $(U, \calS)$, and a set cover problem with a private target subset $R\subset U$ and budget $k$, it is easy to see that the objective
    \begin{align*}
        |R \cap (\cup_{T \in \calT} T)| &= \sum_{e \in R} 1_{e \in T}.
    \end{align*}
    is a decomposable submodular function with $|R|$ summands.
\end{lemma}

\PSSMLB*

\begin{proof}
    We construct a hard instance for maximum coverage. Let $(U, \calS)$ be a set system where $\calS$ consists of all the singletons in $U$. Let $A$ be a set of size $k$ picked uniformly at random from $U$, and let the data set $D_A = A \times [L]$ for $L = \frac{\ln c \frac{e^{\varepsilon}-1}{\delta}}{2 \varepsilon}$. Let $n := |D_A| = |A| \cdot |L|$. Let $\calT$ be $k$ subsets of $\calS$ picked by the solver $M$ . The objective we are trying to maximize is
    \begin{align*}
        f(\calT) = \sum_{e \in D_A} 1_{\{ e \} \in \calT}.
    \end{align*}
    Let $M$ be any $(\varepsilon, \delta)$-DP algorithm for the set cover problem and let $\phi = \Ex_{M, A}[ (M(D_A) \cap A)/|A|]$, i.e. $\phi$ is the average fraction of points of $A$ (and consequently $D_A$) that were recovered successfully by $M$. $\phi$ captures the average approximation factor achieved by the algorithm $M$ over this family of hard instances. 
    
    We see that since $A$ is of size $k$, and the data set $D_A$ is simply points of $A$ repeated with multiplicity, the collection of sets $\calT = \{\{i\} : i \in A\}$ is a solution for this maximization problem that achieves $f(\OPT) = n$.  
    
    We observe that
    \begin{align*}
        \phi &= \Ex_{A, M} \left[\sum_{e \in D_A} 1_{\{ e \} \in \calT}\right]/|A| \\
        &= \Ex_{A, M} \Ex_{i\in A} [1_{i \in M(D_A)}]  \\
        &= \Ex_{i \in U} \Ex_{A,M} [1_{i \in M(D_A)} | i \in A].
    \end{align*}
    
    Fixing any choice of $i\in A$, let $i'$ be uniformly random in $U \backslash A$, and let $A' = (A \backslash \{i\}) \cup \{i'\}$; $A'$ is hence uniformly random over $U \backslash \{i\}$. We see that there is a chain of sets $D_{A}^0, D_A^1, \dots, D_A^L$ such that $D_A^0 = D_A$, $D_A^{t} = (D_A^{t-1} \backslash \{i\}) \cup \{i'\}$ for $t \in [L]$, and $D_A^L = D_{A'}$ (we recall that we treat data sets as multisets, allowing us to swap one copy of $i$ for one copy of $i'$ at a time. Since $M$ is $(\varepsilon, \delta)$-DP, it follows that for all $t\in [L]$,
    \begin{align*}
        \Ex_{M} [1_{i \in M(D_A^t)}] &\geq \exp(-2\varepsilon) \Ex_{M} [1_{i \in M(D_A^{t-1})}] - 2 \delta.
    \end{align*}
    It follows that
    \begin{align*}
        \Ex_{M} [1_{i \in M(D_A')}] &\geq \exp(-2 \varepsilon) \Ex_{M} [1_{i \in M(D_A^{L-1})} ] - 2 \delta. \\
        &\geq \exp(-4\varepsilon) \Ex_{M} [1_{i \in M(D_A^{L-2})} ] - \exp(- 2\varepsilon)\cdot 2\delta - 2 \delta \\
        &\geq \dots \\
        &\geq \exp(-2L \varepsilon) \Ex_{M} [1_{i \in M(D_A^0)}] - 2\delta\left(1 + \exp(-2\varepsilon) + \exp(-4\varepsilon) + \dots\right ) \\
        &\geq \exp(-2L \varepsilon) \Ex_{M} [1_{i \in M(D_A)}] - \frac{2 \delta}{1 - e^{ - 2 \varepsilon}} \\
        &\geq \exp(-2L \varepsilon) \Ex_{M} [1_{i \in M(D_A)}] - \frac{2 \delta}{e^{2 \varepsilon}-1}.
    \end{align*}
    Taking the expectation over $i\in U$ and the randomness in the choice of $A$, we get
    \begin{align*}
        \Ex_{i\in U} \Ex_{A,M} [1_{i \in M(D_A)} | i\not\in A ] &\geq \phi \exp(-2 L \varepsilon) - \frac{2\delta}{e^{2\varepsilon} - 1}.
    \end{align*}
    It follows by the law of total expectation that
    \begin{align*}
        \Ex_{i \in U} \Ex_{A, M} [1_{i \in M(D_A)}] \geq \phi \exp(-2L\varepsilon) - \frac{2\delta}{e^{2\varepsilon} - 1}.
    \end{align*}
    The LHS is at most $k/n$, so rearranging terms we get
    \begin{align*}
        \left(\frac{k}{n} + \frac{2\delta}{e^{2\varepsilon} - 1}\right) \exp(\varepsilon \cdot 2L) &\geq \phi.
    \end{align*}
    It follows that for $n \geq k \frac{e^{2\varepsilon} - 1}{2\delta}$, and $L \leq \frac{1}{2\varepsilon} \log c\frac{e^{2\varepsilon} - 1}{8\delta}$, $\phi$ is at most $c/2$. It follows that for all $c\geq \frac{8\delta}{e^{2\varepsilon}-1}$ either the algorithm fails to achieve the multiplicative approximation factor of $c$, or it incurs additive error $ckL/2 = \Omega((ck/\varepsilon) \log (\varepsilon/\delta))$.
\end{proof}
\section{Experiments}
\label{sec:exp}

\begin{figure}
    \begin{subfigure}{.25\textwidth}
        \begin{tikzpicture}
        \begin{axis}[
                  xlabel = {Cardinality constraint $k$},
                  ylabel = {Clustering cost},
                  error bars/y dir=both, 
                  error bars/y explicit,  
                  scale=0.3,
                  legend pos=outer north east
                  ]
                \addplot table [x=Params, y=Laplace, y error=LaplaceEB, col sep=comma] {graphs/taxi_eps_1E-1.csv};
                \addlegendentry{Laplace}
                \addplot table [x=Params, y=Ours, y error=OursEB, col sep=comma] {graphs/taxi_eps_1E-1.csv};
                \addlegendentry{Gumbel}
                \addplot table [x=Params, y=Non-private, y error=Non-privateEB, col sep=comma] {graphs/taxi_eps_1E-1.csv};
                \addlegendentry{Non-priv.}
                \addplot table [x=Params, y=Random, y error=RandomEB, col sep=comma] {graphs/taxi_eps_1E-1.csv};
                \addlegendentry{Random}
                \end{axis}
        \end{tikzpicture}
        \caption{$\varepsilon = 0.1$, Uber data set}
        \label{fig:variedUber1E-1}
        \end{subfigure}
        \hspace{2in}
        \begin{subfigure}{.25\textwidth}
            \begin{tikzpicture}
            \begin{axis}[
                      xlabel = {Cardinality constraint $k$},
                      ylabel = {Clustering cost},
                      error bars/y dir=both, 
                      error bars/y explicit,  
                      scale=0.3,
                      legend pos=outer north east
                      ]
                    \addplot table [x=Params, y=Laplace, y error=LaplaceEB, col sep=comma] {graphs/taxi_eps_1E0.csv};
                    \addlegendentry{Laplace}
                    \addplot table [x=Params, y=Ours, y error=OursEB, col sep=comma] {graphs/taxi_eps_1E0.csv};
                    \addlegendentry{Gumbel}
                    \addplot table [x=Params, y=Non-private, y error=Non-privateEB, col sep=comma] {graphs/taxi_eps_1E-1.csv};
                    \addlegendentry{Non-priv.}
                    \addplot table [x=Params, y=Random, y error=RandomEB, col sep=comma] {graphs/taxi_eps_1E-1.csv};
                    \addlegendentry{Random}
                    \end{axis}
            \end{tikzpicture}
        \caption{$\varepsilon = 1$, Uber data set}
        \label{fig:variedUber1E0}
    \end{subfigure}\\
    \begin{subfigure}{.25\textwidth}
        \begin{tikzpicture}
                \begin{axis}[
                  xlabel = {Cardinality constraint $k$},
                  ylabel = {Clustering cost},
                  error bars/y dir=both, 
                  error bars/y explicit,  
                  scale=0.3,
                  legend pos=outer north east
                  ]
                \addplot table [x=Params, y=Laplace, y error=LaplaceEB, col sep=comma] {graphs/synth_eps_1E-1.csv};
                \addlegendentry{Laplace}
                \addplot table [x=Params, y=Ours, y error=OursEB, col sep=comma] {graphs/synth_eps_1E-1.csv};
                \addlegendentry{Gumbel}
                \addplot table [x=Params, y=Non-private, y error=Non-privateEB, col sep=comma] {graphs/synth_eps_1E-1.csv};
                \addlegendentry{Non-priv.}
                \end{axis}
        \end{tikzpicture}
        \caption{$\varepsilon = 0.1$, Synthetic data set}
        \label{fig:variedSynth1E-1}
        \end{subfigure}
    \hspace{2in}
    \begin{subfigure}{.25\textwidth}
        \begin{tikzpicture}
                \begin{axis}[
                  xlabel = {Cardinality constraint $k$},
                  ylabel = {Clustering cost},
                  error bars/y dir=both, 
                  error bars/y explicit,  
                  scale=0.3,
                  legend pos=outer north east
                  ]
                \addplot table [x=Params, y=Laplace, y error=LaplaceEB, col sep=comma] {graphs/synth_eps_1E0.csv};
                \addlegendentry{Laplace}
                \addplot table [x=Params, y=Ours, y error=OursEB, col sep=comma] {graphs/synth_eps_1E0.csv};
                \addlegendentry{Gumbel}
                \addplot table [x=Params, y=Non-private, y error=Non-privateEB, col sep=comma] {graphs/synth_eps_1E-1.csv};
                \addlegendentry{Non-priv.}
                \end{axis}
        \end{tikzpicture}
        \caption{$\varepsilon = 1$, Synthetic data set}
        \label{fig:variedSynth1E0}
    \end{subfigure}
    \caption{We compare the performance of \Cref{alg:PSSM} with Laplace and Gumbel noise, as well as the non-private \Cref{alg:nonPrivStream} and a trivial random selection algorithm (marked Random in the legend) for a submodular maximization problem derived from an instance of $k$-medians clustering. This comparison is done for two different choices of privacy parameter $\varepsilon=0.1$ and $\varepsilon=1$; $\delta = 1/|P|^{1.5}$ where $P$ is the private data set to be protected. We also experiment with two different data sets, the Uber data set \citep{fivethirtyeight_2019} and a synthetic data set generated as described in the text.}
    \label{fig:exp}
\end{figure}
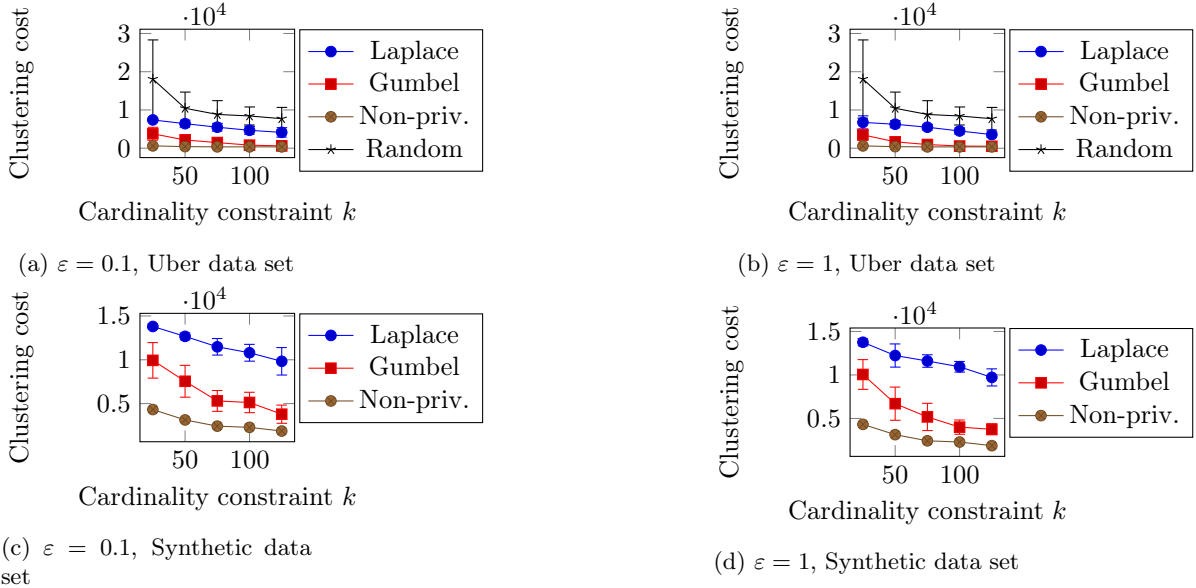

In this section, we empirically evaluate our approach on a practical instance of private submodular maximization. Given a set of points $V$, a metric $d: V \times V \rightarrow \mathbb{R}$, a private set of demand points $P\subseteq V$, the objective of the $k$-medians problem is to select a set of points $S \subset V$, $|S| \leq k$ to minimize $\text{cost}(S) = \sum_{p \in P} d(p,S)$, where $d(p,S) = \min_{s \in S} d(v,s)$. An application of his problem is allocating a relatively few ($k$) service centers to be able to reach a large set of clients ($P$) and ensure that there is at least one service location not too far from most clients; when the clients' locations are private but the service locations are public, this problem requires a differentially private solver. We can optimize this objective by casting it into the following submodular maximization problem:  $\max_{S\subset V,|S|\le k}\sum_{p\in P} 1 - d(p,S)/\G$, where $G$ is a normalization constant so that  $f_p(S) = 1 - d(p,S)/\G \in [0,1]$. Setting $d(p,\emptyset) = \G$, it can be checked that $\text{cost}(S)$ is a monotone decomposable submodular function.

We compare the performance of \Cref{alg:PSSM} with \textbf{Laplace} and \textbf{Gumbel} noise on two data sets. First, following \cite{DBLP:conf/icml/MitrovicB0K17,DBLP:journals/corr/abs-2005-14717} we use the Uber data set \cite{fivethirtyeight_2019} of Uber cab pick ups in Manhattan for the month of April 2014; the goal is to allocate public waiting locations for Uber cabs so as to serve requests from clients within short distances. Second, we construct a synthetic dataset in $\mathbb{R}^2$ by generating clients $P$ from a mixture of $50$ Gaussian distributions, each with identity covariance matrix and mean chosen uniformly at random from a bounding box of $[20]\times[20]$. We sample a $1000$ points from each Gaussian for a total of $50,000$ clients. For both settings, we set $d(\cdot, \cdot)$ to be the $\ell_1$ or Manhattan distance, i.e. $d(a, b) = |a_1 - b_1| + |a_2 - b_2|$. We set V to be a $50\times 50$ 2-D grid of points uniformly spanning the rectangular domain.

We compare our two algorithms with an approach that selects $k$ \textbf{Random} points from the stream as a differentially private baseline, and the \textbf{Non-private} algorithm~\ref{alg:nonPrivStream}. For both data sets, we set $\delta = 1/|P|^{1.5}$ and $\theta = 0.2$.  In \cref{fig:exp} we graph the clustering cost versus the cardinality constraint $k$ on the Taxi and Synthetic data sets. We measure and report the mean and standard deviation of the clustering cost over $20$ random runs with varied $k$ and $\varepsilon$ (performed on a PC with 5.2 GHz i9 chip and 64 GB RAM). 

For our private algorithms, we set  $E = \min\left( k\log n/\varepsilon, |P|/2\right)$.  We set $E = \min(\max_{e_i \in V} f(e_i), k\log n/\varepsilon,\allowbreak |P|/2) $ for the non-private algorithm. This guarantees that the number of copies for the non-private algorithm is at least that of the private algorithms. Instead of using the exponential mechanism to output the solution in algorithm~\ref{alg:PSSM}, we use the Report Noisy Max mechanism with equivalent privacy guarantee and similar tail bound (see \cite{DBLP:journals/fttcs/DworkR14}); this avoids potential overflow issues with the exponential mechanism. When the number of elements left in the stream of a non-private instance is less than $k-|S|$, we add the rest of the points to $S$. This does not affect the theoretical guarantee, but might benefit the algorithm empirically.

Although we apply advanced composition in our theoretical analyses as it asymptotically requires lower noise than basic composition, because of the difference in constant coefficients, basic composition works better for the number of thresholds we need to consider. For this experiment, we apply basic composition and set $\varepsilon' = \varepsilon/T, \delta' = \delta/T$. 

In \cref{fig:variedSynth1E-1,fig:variedSynth1E0} we omit the result of the random approach, which has much higher cost than the rest. In \cref{fig:variedUber1E-1,fig:variedSynth1E-1}, we report the results for $\varepsilon = 0.1$.
We observe that Gumbel noise outperforms Laplace noise in both settings. Although Gumbel noise leads to higher variance, increasing $k$ results in greater reduction of the clustering cost than the Laplace noise approach. We observe similar results in  \cref{fig:variedUber1E0,fig:variedSynth1E0} for $\varepsilon = 1$. Gumbel noise continues to outperform Laplace noise in both settings. Increasing the privacy budget from $0.1$ to $1$ slightly improves the utility of differentially private approaches.

\bibliographystyle{unsrtnat}
\bibliography{biblio}

\end{document}